\documentclass{article}

\usepackage[english]{babel}

\usepackage[letterpaper,top=2cm,bottom=2cm,left=3cm,right=3cm,marginparwidth=1.75cm]{geometry}

\usepackage{amsmath}
\usepackage{graphicx}
\usepackage[colorlinks=true, allcolors=blue]{hyperref}

\usepackage{microtype}
\usepackage{graphicx}
\usepackage{subcaption}
\usepackage{booktabs} % for professional tables
\usepackage{algorithm}
\usepackage{algorithmic}

\usepackage{hyperref}

\usepackage{amsmath}
\usepackage{amssymb}
\usepackage{mathtools}
\usepackage{amsthm}
\usepackage{dsfont}

\DeclareMathOperator*{\argmax}{arg\,max}

\usepackage{dsfont}
\usepackage{xcolor}

\usepackage{caption}
\usepackage[capitalize,noabbrev]{cleveref}

\theoremstyle{plain}
\newtheorem{theorem}{Theorem}[section]
\newtheorem{proposition}[theorem]{Proposition}
\newtheorem{lemma}[theorem]{Lemma}

\theoremstyle{definition}
\newtheorem{definition}[theorem]{Definition}

\theoremstyle{remark}

\usepackage[textsize=tiny]{todonotes}

\usepackage{natbib}  % DO NOT CHANGE THIS AND DO NOT ADD ANY OPTIONS TO IT
\usepackage{caption} % DO NOT CHANGE THIS AND DO NOT ADD ANY OPTIONS TO IT

\usepackage{multirow}

\title{Rethinking Divisive Hierarchical Clustering from \\ a Distributional Perspective}
\author{Kaifeng Zhang, Kai Ming Ting, Tianrun Liang, Qiuran Zhao}

\begin{document}
\maketitle

\begin{abstract}
We uncover that current objective-based Divisive Hierarchical Clustering (DHC) methods produce a dendrogram that does not have three desired properties i.e., no unwarranted splitting, group similar clusters into a same subset, ground-truth correspondence. This shortcoming has their root cause in using a set-oriented bisecting assessment criterion. We show that this shortcoming can be addressed by using a distributional kernel, instead of the set-oriented criterion; and the resultant clusters achieve a new distribution-oriented objective to maximize the total similarity of
all clusters (TSC). 
%The proposed new DHC method builds a dendrogram by bisecting via a distributional kernel. 
Our theoretical analysis shows that the resultant dendrogram guarantees a lower bound of TSC. The empirical evaluation shows the effectiveness of our proposed method on artificial and Spatial Transcriptomics (bioinformatics) datasets. Our proposed method successfully creates a dendrogram that is consistent with the biological regions in a Spatial Transcriptomics dataset, whereas other contenders fail.
\end{abstract}

\section{Introduction}
Hierarchical clustering (HC) seeks to explore the relationships between clusters at various levels of granularity. HC is extensively applied \cite{li2019community,tumminelloCorrelationHierarchiesNetworks2010,pang2022hierarchical,sangaiah2022hierarchical,diezpalacioNovelBrainPartition2015,fu2021hierarchical} due to its hierarchical structure named dendrogram, which uncovers the pattern of clusters and their subclusters in a given dataset.

HC is categorized into two types: Agglomerative Hierarchical Clustering (AHC) and Divisive Hierarchical Clustering (DHC)~\cite{cohen2019hierarchical}. AHC works in a bottom-up manner. It initially groups the closest pair of points according to a linkage function and progressively merges the two nearest subclusters until all the points are in one cluster. DHC works in a top-down manner. It starts with all points in one cluster and successively partitions them into smaller clusters according to some bisecting assessment criterion.

The core problem in objective-based DHC in building a dendrogram $T$ is how to bisect a set of data points into two subsets at each internal node of $T$. By \emph{treating each cluster as a set of points}, it requires a set-oriented bisecting assessment criterion to perform each bisecting. We show here that an optimized dendrogram, though meeting the stated objective of a DHC method, can be a poor dendrogram that does not possess certain desirable properties. We show that this is because the entire process ignores the distributional information in a dataset.

%based on a function $h$ which works as a criterion that each split tries to optimize.
%This optimization step involves considerable time expense. In addition, $h$ is point-oriented since it assesses the contribution of each point in either of the two subsets in each split. It is unable to examine the closeness of clusters within each subset because clusters have yet to be identified during the dendrogram construction process. As a consequence, this point-oriented approach may yield a haphazard hierarchy where each split does not ensure that similar clusters are grouped together in one subset which is dissimilar to those in the other subset.
% To our knowledge, all existing DHC methods work within this \emph{point-oriented paradigm} to find ways to reduce the high time complexity by modeling the split as an optimization problem.

As the root cause is due to the use of a \emph{set-oriented} bisecting assessment criterion, 
we propose an alternative \emph{distribution-oriented} approach where each cluster is treated as independent and identically distributed (i.i.d.) points generated from a distribution (which characterizes the shape, size, and density of the ground-truth cluster). %derived from a given dataset of $n$ points. 
The resultant objective-based DHC method is derived from a distributional kernel, and it succeeds in producing a dendrogram with three desired properties (see the next paragraph),
%(i.e., no unwarranted cluster splitting, group similar clusters into a same subset, ground-truth correspondence), 
and the clusters produced satisfy a new distribution-oriented objective. 
%whereas a recent objective-based DHC method of the point-oriented paradigm fails to do so, even though the dendrogram has achieved its stated objective.
In addition, we show that this new approach has $O(n)$ time complexity.

The proposed new objective-based DHC has the following unique features: (i) It is the first objective-based DHC, named H-$\mathcal{K}C$, which has the objective function defined based on a distributional kernel $\mathcal{K}$. (ii) The bisecting employs the same distributional kernel as used in the objective function. This is unlike the current objective-based DHC in two aspects. First, our proposed method does not need a set-oriented bisecting assessment criterion. Second, the set-oriented bisecting assessment criterion used in existing DHC methods has no direct relationship to its objective function. (iii) H-$\mathcal{K}C$ runs in linear time. In contrast, the current fastest objective-based DHC has nearly linear time only if a similarity graph is given (which often requires quadratic time to construct). (iv) It is a distribution-oriented algorithm that discovers clusters of arbitrary shapes, sizes and densities. 
% \textcolor{blue}{While SpecWRSC is also distribution-based, its sparse-cut criterion produces unbalanced dendrogram. In comparison, H-$\mathcal{K}C$ ensures a better dendrogram structure quality,i.e., similar clusters within a subset and a more balanced dendrogram.} 
(v) The dendrogram produced by H-$\mathcal{K}C$ possesses \textbf{three desired properties: no unwarranted cluster splitting, group similar clusters into a same subset, ground-truth correspondence} (see Section \ref{sec-Desired-Properties} for details), whereas other existing methods fail to do so.

Our contributions are: 
\begin{enumerate}
    \item Identifying the root cause of the shortcoming of existing objective-based DHC: the use of a set-oriented bisecting assessment criterion.
    \item Creating a new distribution-oriented approach to the objective-based DHC problem.
    \item Proposing a new objective-based DHC algorithm using the distribution-oriented approach called H-$\mathcal{K}C$ that produces a dendrogram with three desired properties.
    \item Establishing the theoretical guarantee that H-$\mathcal{K}C$ is an objective-based DHC, and it produces a dendrogram which has lower bounded a global objective  (i.e., the dendrogram is guaranteed to have a certain quality). 
%(C) DHC and AHC versions of the proposed algorithm using the same distributional kernel as the means to split and merge, respectively, produce similar dendrograms. 
%(ii) The hierarchical and flat clustering versions  H-$\mathcal{K}BC$ and F-$\mathcal{K}BC$ produce the same set of $k$ clusters. 
% (iii)  The new and the existing objection functions are equivalent under some condition (to be verified) 
    \item Conducting empirical evaluations comparing different objective-based DHC algorithms on artificial and Spatial Transcriptomics (bioinformatics) datasets.
\end{enumerate}
\section{Related Work}
Early DHC methods produce a dendrogram that does not base on an objective function. Examples of non-objective-based methods are as follows.  Two earliest methods perform bisecting based on seed selection \cite{macnaughton1964dissimilarity,hubert1973monotone}. PDDP \cite{boley1998principal} uses a linear boundary to split the data in one principal direction. DIANA \cite{DIANA-1990} performs bisecting by mimicking the way a political party might split up due to inner conflicts. It separates a point which is least similar to its own cluster as the initial seed of a new cluster, and then it moves points from the old cluster to the new cluster if they are more similar to this new one. 

An early DHC called Bisect-Kmeans \cite{Bisecting-k-means-2000} performs optimization on each split. It is later identified to be an objective-based DHC by using Kmeans to select the best split with the least total sum of squared errors. \citealp{wang2020objective} proves that Bisect-Kmeans approximates the Revenue Objective of a given dataset $X$, which is defined to be $rev_T(X)=\sum_{S\rightarrow{} (S_1,S_2)\in T}\ \sum_{x\in S_1,y\in S_2} rev(x,y)$, where $S\rightarrow{} (S_1,S_2)\in T$ means all the splits in dendrogram $T$, $rev(x,y)=min\{\frac{d(x,y)}{max\{d(x,\rho(S_1),d(y,\rho(S_2)\}},1\}$, $\rho(S)$ is the center of $S$, and $d$ is Euclidean distance. 

Other works on objective-based  hierarchical clustering \cite{dasgupta2016cost,roy2017hierarchical,chatziafratis2018hierarchical,ghoshdastidar2019foundations}  focus on producing a dendrogram $T$ that has the smallest Dasgupta cost function. Given an undirected graph $\mathbf{G}=(V,E,w)$ with $|V|$ vertices, $|E|$ edges connecting two vertices $u, v \in V$, and similarity function $w : V \times V \rightarrow{\mathbb{R}_{\ge 0}}$, the Dasgupta cost function of $T$ produced from $\mathbf{G}$ is given as follows \cite{dasgupta2016cost}: $cost_\mathbf{G}(T) = \sum_{\{u,v\} \in E} w(u,v) \times |LCA_T(u,v)|$,
where $LCA_T(u,v)$ denotes the lowest common ancestor of  $u$ and $v$ in $T$. A variant of Dasgupta cost function is CKMM objective \cite{cohen2017hierarchical,naumov2021objective}, which basically replaces similarity $w(u,v)$ with distance $d(u,v)$; and minimizing Dasgupta cost is replaced with maximizing CKMM.

% The reason that current methods f

% \citeauthor{wang2020objective} follow Biect-Kmeans and proves that Bisect-Kmeans approximates the Revenue Objective, which is defined to be
% \begin{equation*}
%     rev_T(V)=\sum_{S\rightarrow{} (S_1,S_2)\in T}\sum_{i\in S_1,j\in S_2} rev(i,j),
% \end{equation*}
% where $S\rightarrow{} (S_1,S_2)\in T$ means all the splits in dendrogram $T$, $rev(i,j)=min\{\frac{d(i,j)}{max\{d(i,\rho(S_1),d(j,\rho(S_2)\}},1\}$, and $\rho(S)$ is the center of $S$.

% A recent work SpecWRSC \cite{laenen2023nearly} produces an $\mathcal{O}(1)$-approximate tree of Dasgupta cost function in nearly-linear time relative to the size of the input graph (excluding the cost of building the adjacency matrix of the graph). This is achieved by using the sparsest-cut criterion which focuses solely on minimizing the sparsity between two subsets, i.e.,  $Sparsity(Y)=\frac{W(Y,V\setminus Y)}{|Y|\cdot|V\setminus Y|}$, where $W(Y,Z) = \sum_{u \in Y, v \in Z} w(u,v)$. 

% \textcolor{blue}{SpecWRSC \cite{laenen2023nearly} produces an $\mathcal{O}(1)$-approximate tree of Dasgupta cost using a sparsest-cut criterion $Sparsity(Y)=\frac{W(Y,V\setminus Y)}{|Y|\cdot|V\setminus Y|}$.} While effective for the global objective, it tends to separate a cluster from the rest at each split (Figure \ref{subfig: demo SpecWRSC}), failing to group similar clusters.

A recent work SpecWRSC \cite{laenen2023nearly} uses Spectral Clustering to produce a set of clusters and then builds a dendrogram from these clusters. It produces an $\mathcal{O}(1)$-approximate tree of Dasgupta cost function in nearly-linear time relative to the size of the input graph (excluding the cost of building the adjacency matrix of the graph). This is achieved by using the set-oriented bisecting assessment criterion (sparsest-cut) which focuses solely on minimizing the sparsity between two subsets, i.e.,  $Sparsity(Y)=\frac{W(Y,V\setminus Y)}{|Y|\cdot|V\setminus Y|}$, where $W(Y,Z) = \sum_{u \in Y, v \in Z} w(u,v)$. While effective for the global objective, it tends to separate a cluster from the rest at each split (Figure \ref{subfig: demo SpecWRSC}), failing to group similar clusters.
% \textcolor{blue}{In other words, SpecWRSC focuses on approximating the optimal objective; and although a minimum-cost dendrogram may sound like a reasonable objective, it produces a dendrogram with a poor structure (see an example with unbalanced dendrogram in Figure \ref{fig: example} in the next section) that always separating a cluster from the rest of the clusters at each split. As a result, it can not ensure that similar clusters are grouped into the same subset at each split. 
% % This is because it treats clusters as sets, failing to adequately capture the similarity between clusters.
% }

Current methods assess the quality of a split based on either the cohesiveness of each individual subset (e.g., Bisect-Kmeans) or the dissimilarity between the two resulting subsets (e.g., SpecWRSC), i.e. set-oriented bisecting assessment criterion. They can not consider the distributional information inside the subsets.

In contrast, our proposed method uses a distributional kernel to build a dendrogram such that similar clusters are in either subset of each split (i.e., no similar clusters are split into different subsets). Its objective function is the total similarity of all clusters (TSC) in terms of the distributional kernel, which is used for the first time in hierarchical clustering.
This function is similar to that used in an existing flat clustering method called Point-Set Kernel Clustering (psKC) \cite{pskc}, but psKC's objective function is TSC in terms of point-set kernel. The difference between the two kernels and the details of psKC are provided in Appendix \ref{append: ablation}.
% Interestingly, both psKC and the proposed method are \emph{objective based without the need of optimization}.

\begin{figure*}[t]
\centering
     \begin{subfigure}
        {.24\linewidth}
      \centering   
\includegraphics[width=\linewidth]{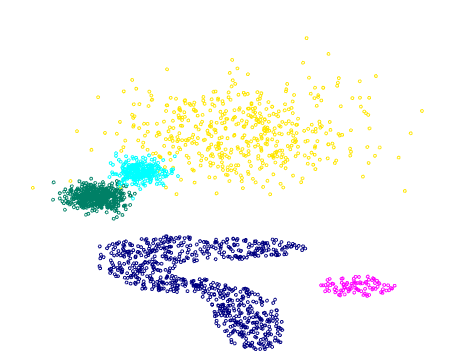}
        \caption{Ground-truth plot}
        \label{subfig:GT}
    \end{subfigure}
     \begin{subfigure}
        {.24\linewidth}
      \centering   
\includegraphics[width=\linewidth]{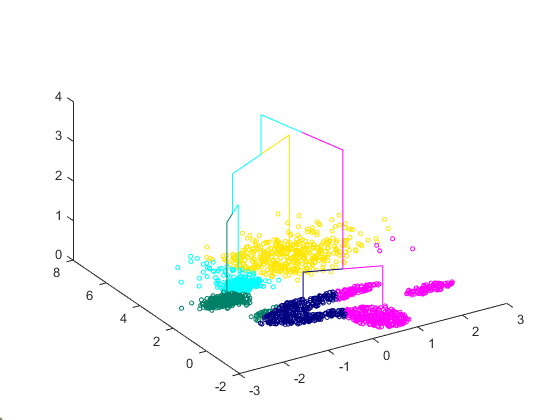}
        \caption{Bisect-Kmeans: $\wp=0.86$}
        \label{subfig: demo Bisect-Kmeans}
    \end{subfigure}
\begin{subfigure}
        {.24\linewidth}
      \centering   
\includegraphics[width=\linewidth]{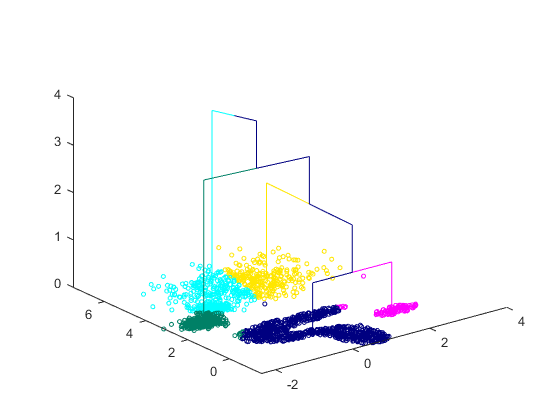}
        \caption{SpecWRSC: $\wp=0.83$}
        \label{subfig: demo SpecWRSC}
    \end{subfigure} 
   \begin{subfigure}
        {.24\linewidth}
      \centering   
\includegraphics[width=\linewidth]{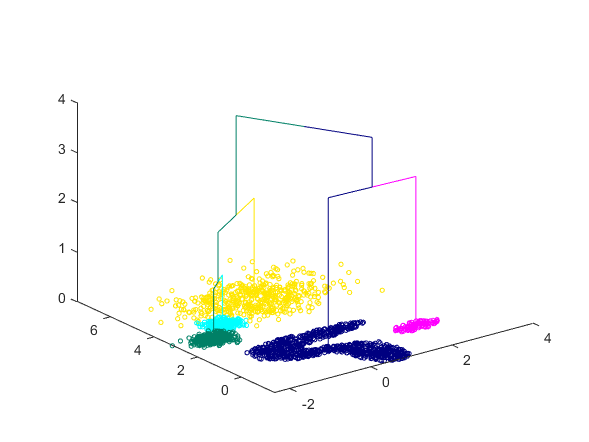}
        \caption{H-$\mathcal{K}C$: $\wp=0.98$}
        \label{subfig: demo HKBC}
    \end{subfigure} 
    \caption{Dendrograms of Bisect-Kmeans, SpecWRSC and H-$\mathcal{K}C$ on the artificial dataset which contains clusters of arbitrary shapes and varied densities. $\wp$ is Dendrogram Purity \protect\cite{heller2005bayesian} (see the details in Appendix \ref{appendix: purity}).}
\label{fig: example}
\end{figure*}

%Although it uses a point-set kernel, we employ a version of psKC, that uses a distributional kernel instead, to find a set of cluster cores from a small sample of the a given dataset

%is known for its ability of detecting clusters of arbitrary shapes, varied densities and sizes. In this paper, we develop a new method which uses distributional kernel to build a dendrogram that ensures similar clusters are in either subset of each split. The objective function of psKC, i.e., the total similarity of all clusters (TSC), is used for the first time in hierarchical clustering (see Section \ref{sec-H-KC} for details). TSC is defined as $\sum_{i=1}^{k} \sum_{x \in C_i} \mathcal{K}(\delta(x),\mathcal{P}_{C_i} )$, where $C_i$ is cluster $i$, $\mathcal{K}$ is distributional kernel, $\delta$ is Dirac measure and $\mathcal{P}_{C_i}$ is the distribution of $C_i$. The detailed description of psKC is provided in the Appendix.} 

Current evaluation methods on dendrograms have largely focused on minimizing cost and employed the commonly used Dendrogram Purity (DP) \cite{heller2005bayesian,kobren2017hierarchical}. Yet, it  has been revealed that `perfect DP can be achieved when each ground-truth cluster
corresponds to a subtree, regardless of hierarchy on top or
inside of the subtrees' \cite{naumov2021objective}. 
%measures the goodness of the match of the clusters in the leaf nodes of a dendrogram wrt the ground-truth clusters, irrespective of the structure of the dendrogram \cite{naumov2021objective}. 
In other words, DP does not reflect the quality of a dendrogram. Hence, in addition to DP, we compare the structures of dendrograms produced by different algorithms through visualization. 

% \begin{figure*}[htb]
% % \centering
%      \begin{subfigure}
%         {.245\linewidth}
%       \centering   
% \includegraphics[width=\linewidth]{pics_arxiv/obj2gt_2.png}
%         \caption{Ground-truth plot}
%         \label{subfig:GT}
%     \end{subfigure}
%      \begin{subfigure}
%         {.245\linewidth}
%       \centering   
% \includegraphics[width=\linewidth]{pics_arxiv/obj2_bisectkm_0.86.png}
%         \caption{Bisect-Kmeans: $\wp=0.86$}
%         \label{subfig: demo Bisect-Kmeans}
%     \end{subfigure}
% \begin{subfigure}
%         {.245\linewidth}
%       \centering   
% \includegraphics[width=\linewidth]{pics_arxiv/obj2_WRSC_0.827.png}
%         \caption{SpecWRSC: $\wp=0.83$}
%         \label{subfig: demo SpecWRSC}
%     \end{subfigure} 
%    \begin{subfigure}
%         {.245\linewidth}
%       \centering   
% \includegraphics[width=\linewidth]{pics_arxiv/obj2_HKC_0.98.png}
%         \caption{H-$\mathcal{K}C$: $\wp=0.98$}
%         \label{subfig: demo HKBC}
%     \end{subfigure} 
%     \caption{Dendrograms of Bisect-Kmeans, SpecWRSC and H-$\mathcal{K}C$ on the artificial dataset which contains clusters of arbitrary shapes and varied densities. $\wp$ is Dendrogram Purity \protect\cite{heller2005bayesian} (see the details in Appendix \ref{appendix: purity}).}
% \label{fig: example}
% \end{figure*}

\section{Desired Properties of a Dendrogram}
\label{sec-Desired-Properties}

Compared to flat clustering, a dendrogram provides the nested relationship between a cluster and its subclusters. The desired properties of a good dendrogram are:
\begin{enumerate}
    \item \textbf{No unwarranted splitting}: A cluster should not be split into two sub-clusters in child nodes of a dendrogram if the cluster is not already at one node by itself.
    \item \textbf{Group similar clusters into a same subset}: Every split in a dendrogram shall produce two subsets such that each subset contains clusters which are closer to each other than clusters in the other subset.
    \item \textbf{Ground-truth correspondence}: Each leaf node of a dendrogram corresponds to a single cluster in a given dataset. This hierarchical clustering outcome is equivalent to the best outcome of a flat clustering.
\end{enumerate}

\subsection{Current Methods Fail to Yield a Dendrogram with All Desired Properties Mainly Due to The Set-Oriented Bisecting Assessment Criterion Employed}
% Shortcomings of Current Methods Are Due to The Point-Oriented Criterion Employed.
\label{sect:distribution}
It is interesting to note that existing objective-based DHC methods produce a dendrogram which does not possess all the three properties, described in the following paragraphs.

% \begin{figure*}[htb]
% % \centering
%      \begin{subfigure}
%         {.245\linewidth}
%       \centering   
% \includegraphics[width=\linewidth]{pics_arxiv/obj2gt_2.png}
%         \caption{Ground-truth plot}
%         \label{subfig:GT}
%     \end{subfigure}
%      \begin{subfigure}
%         {.245\linewidth}
%       \centering   
% \includegraphics[width=\linewidth]{pics_arxiv/obj2_bisectkm_0.86.png}
%         \caption{Bisect-Kmeans: $\wp=0.86$}
%         \label{subfig: demo Bisect-Kmeans}
%     \end{subfigure}
% \begin{subfigure}
%         {.245\linewidth}
%       \centering   
% \includegraphics[width=\linewidth]{pics_arxiv/obj2_WRSC_0.827.png}
%         \caption{SpecWRSC: $\wp=0.83$}
%         \label{subfig: demo SpecWRSC}
%     \end{subfigure} 
%    \begin{subfigure}
%         {.245\linewidth}
%       \centering   
% \includegraphics[width=\linewidth]{pics_arxiv/obj2_HKC_0.98.png}
%         \caption{H-$\mathcal{K}C$: $\wp=0.98$}
%         \label{subfig: demo HKBC}
%     \end{subfigure} 
%     \caption{Dendrograms of Bisect-Kmeans, SpecWRSC and H-$\mathcal{K}C$ on the artificial dataset which contains clusters of arbitrary shapes and varied densities. $\wp$ is Dendrogram Purity \protect\cite{heller2005bayesian} (see the details in Appendix \ref{appendix: purity}).}
% \label{fig: example}
% \end{figure*}
\textbf{The first desired property depends heavily on the ability to detect clusters of complex shapes, different data sizes and varied densities. And the third property is a natural consequence of the first one}. Kmeans is known for its inability to find clusters of arbitrary shapes. For example, in a dataset with a L-shaped cluster, shown in Figure \ref{subfig: demo Bisect-Kmeans}, Bisect-Kmeans  splits this cluster into two sub-clusters before it is in one node by itself, where the top left part of the L-shaped cluster is split in the first bisecting. 
% SpecWRSC does not possess the first property because the sparsest-cut criterion is a variant of the ratio-cut criterion used in Spectral Clustering \cite{von2007tutorial}, which is known to  have trouble in detecting clusters of varied densities. As shown in Figure \ref{subfig: demo SpecWRSC}, SpecWRSC fails to identify the three Gaussian Distributions: the sparse Gaussian on the right is split and mixed with the dense Gaussian in the middle.

% \textbf{The first desired property depends heavily on the ability to detect clusters of complex shapes, different data sizes and varied densities. And the third property is a natural consequence of the first one}. Kmeans is known for its inability to find clusters of arbitrary shapes. For example, in a dataset with a L-shaped cluster, shown in Figure \ref{subfig: demo Bisect-Kmeans}, Bisect-Kmeans  splits this cluster into two sub-clusters before it is in one node by itself, where the top left part of the L-shaped cluster is split in the first bisecting. 
SpecWRSC relies on Spectral Clustering \cite{von2007tutorial} for initialization, which lacks the first property due to its known difficulty in detecting clusters of varied densities \cite{SC-Limitations-2006}. As a result, it fails to identify the three Gaussian distributions, incorrectly splitting the sparse right cluster and merging it with the dense middle one (Figure \ref{subfig: demo SpecWRSC}).
% \textcolor{blue}{SpecWRSC uses Spectral Clustering for initial clusters and performs later based on these clusters. SpecWRSC does not possess the first property because of the initial Spectral Clustering \cite{von2007tutorial}, which is known to  have trouble in detecting clusters of varied densities. As shown in Figure \ref{subfig: demo SpecWRSC}, SpecWRSC fails to identify the three Gaussian Distributions: the sparse Gaussian on the right is split and mixed with the dense Gaussian in the middle. }

\textbf{The second desired property demands that the bisecting considers the similarity of clusters within each subset of the split.} Conceptually, Bisect-Kmeans utilizes a bisecting assessment criterion that does not take into account the similarity of clusters within each subset of the split because it is set-oriented. Bisect-Kmeans chooses a split which minimizes the total sum of squared distance from the two mean vectors of the two subsets. Though it considers the cohesiveness in each subset, it is examined with respect to the mean vector only. Each of the mean vectors is often a poor representative of different clusters in either subsets of the split. 

% In Figure \ref{subfig: demo SpecWRSC}, SpecWRSC produces a highly unbalanced tree which separates one cluster from the rest at each internal node. In addition, 
% the yellow cluster is less similar to the green and L-shaped clusters than the light blue cluster. Yet, the first split of SpecWRSC separates the light blue cluster and groups the yellow, green and L-shaped clusters together in one subset. 

A key limitation of SpecWRSC is that it yields undesirable dendrograms even with optimal initial clustering (e.g., psKC as shown in Appendix \ref{appendix: SpecWRSC-psKC}). This is inherent to its bisecting assessment criterion, which tends to separate one cluster from the rest at each step, resulting in an unbalanced tree that fails to group similar clusters. This is also reflected in Figure \ref{subfig: demo SpecWRSC} where SpecWRSC produces a highly unbalanced tree which separates one cluster from the rest at each internal node. In addition, 
the yellow cluster is less similar to the green and L-shaped clusters than the light blue cluster. Yet, the first split of SpecWRSC separates the light blue cluster and groups the yellow, green and L-shaped clusters together in one subset.

% \textcolor{blue}{A key limitation of SpecWRSC is that even with an optimal initial clustering (e.g., via psKC) as shown in Appendix \ref{appendix: SpecWRSC-psKC}, it still yields an undesirable dendrogram. This stems from its splitting rule, which tends to separate a single cluster from the rest at each split, resulting in an unbalanced tree that fails to group similar clusters within a same subset. Notably, this issue with the dendrogram is unrelated to the quality of the initial clustering and is instead inherent to SpecWRSC’s splitting mechanism.}

% \textcolor{blue}{This is also reflected in Figure \ref{subfig: demo SpecWRSC} where SpecWRSC produces a highly unbalanced tree which separates one cluster from the rest at each internal node. In addition, 
% the yellow cluster is less similar to the green and L-shaped clusters than the light blue cluster. Yet, the first split of SpecWRSC separates the light blue cluster and groups the yellow, green and L-shaped clusters together in one subset.}

%create cluster with the best the cluster cohesiveness, which is averaged pairwise similarity inside each cluster.
Both Bisect-Kmeans and SpecWRSC rely on set-oriented bisecting assessment criterion to determine the best split that does not ascertain that (a) similar clusters are grouped into a same subset, and (b) no individual clusters are split in any bisecting before each cluster is in one node by itself. In short, the use of set-oriented bisecting assessment criterion, which does not consider the integrity of each cluster and similarity of clusters within a same subset to determine the split, is a fundamental shortcoming.

To address this fundamental shortcoming, we propose to use a criterion which must preserve  each cluster holistically while building a dendrogram. We show that this can be achieved by treating each cluster as a distribution, and using a distributional kernel to perform the bisecting and maintain the integrity of each cluster. 

It has been shown \cite{smola2007hilbert} that the feature map of a distributional kernel represents a distribution well. By treating each cluster as a distribution, the use of a distributional kernel to perform bisecting is the key to producing a dendrogram with the three desired properties---avoiding the shortcoming of current methods.
\section{Preliminary: Distributional Kernel}
Distributional Kernel $\mathcal{K}$ of Kernel Mean Embedding (KME) \cite{smola2007hilbert} is derived using a symmetric and positive definite point kernel $\kappa$ with feature map $\phi$. It represents a distribution $\mathcal{P}$ as a point in Reproducing Kernel Hilbert Space (RKHS) with its feature map $\Phi(\mathcal{P})=\int_{\mathcal{X}}\kappa(x,\cdot)d\mathcal{P}$, where $\mathcal{X}$ is the support of $\mathcal{P}$.

Given two datasets $X$ and $Y$, sampled from distributions $\mathcal{P}_X$ and $\mathcal{P}_Y$, respectively, the similarity between $\mathcal{P}_X$ and $\mathcal{P}_Y$ can be estimated via KME as $\mathcal{K}(\mathcal{P}_X,\mathcal{P}_Y)=\frac{1}{|X||Y|}\sum_{x\in X,y\in Y}\kappa(x,y) =\left \langle\Phi(X),\Phi(Y)\right \rangle$, where the kernel mean map $\Phi(\mathcal{P}_X)=\frac{1}{|X|}\sum_{x\in X}\phi(x)$.

\section{Set-Oriented and Distribution-Oriented Approaches to Bisecting in DHC}

The two approaches to bisecting to produce a dendrogram $T$ in DHC are defined as follows:
\begin{definition}
Set-oriented approach to bisecting: Given a dataset, each bisecting at an internal node of  $T$ splits the set of data points at this node into two subsets based on a bisecting assessment criterion which guides the split by assessing the
quality of the split as the dissimilarity between the two subsets (e.g., SpecWRSC) or the cohesiveness of each subset (e.g., Bisect-Kmeans) after each split.
\label{def-point-oriented}
\end{definition}

% Distribution-oriented approach 1 to bisecting to produce a dendrogram $T$ in DHC:\\
% Given a dataset, identify a set $\mathcal{C}$ of  $k$ core clusters in the first step, and construct a dendrogram with $k$ lead nodes from $\mathcal{C}$.   Each bisecting at an internal node of  $T$ splits the set of core clusters at this node into two subsets based on a function $h$ which assesses the contribution of each cluster in either of the two subsets after
% the split.

\begin{definition}
Distribution-oriented approach to bisecting: Given a dataset, identify a set $C$ of  $k$ core clusters in the first step, and construct a dendrogram $T$ with $k$ leaf nodes from $C$ in the second step.   Each bisecting at an internal node of  $T$ splits the set of core clusters at this node into two subsets based on a distributional kernel which ensures that similar core clusters are grouped into the same subset.
\label{def-distribution-oriented}
\end{definition}

Although the distribution-oriented approach requires an additional step before constructing a dendrogram, the bisecting in the second step is simpler because it is conducted on a set of $k$ core clusters rather than a set of $n$ data points, where $k$ is a small constant and is independent of dataset size $n$.

\begin{definition}
A core cluster consists of core points in a cluster which defines the shape, size, and density of the cluster. 
Non-core or noise points can be eliminated without affecting the cluster being recognized as the same cluster.
\label{def-core-cluster}
\end{definition}

Given Definitions \ref{def-distribution-oriented} \& \ref{def-core-cluster}, we show later that the set of core clusters can be identified relatively easily using a small subset of the given dataset. As a result, it provides the first means to a linear-time algorithm.

It is interesting to note that the distributional kernel in the distribution-oriented approach is not a bisecting assessment criterion through which the splitting tries to optimize. The use of distributional kernel and a heuristic and greedy way to assign core clusters during the split provide the second means to a linear-time algorithm.

% One interesting implication of the elimination of $h$ in the distribution-oriented approach is that the bisecting in building a dendrogram can be achieved without optimization! In fact, we show that \textbf{the proposed no-optimization method produces a better dendrogram than that by an existing optimization-based method}. This is possible because of the distributional treatment---representing each cluster as a distribution via a distributional kernel. The proposed algorithm based on Definition \ref{def-distribution-oriented} is described in the next section.

\section{Hierarchical Clustering based on Distributional Kernel}
\label{sec-H-KC}

\begin{algorithm}[h]
\caption{H-$\mathcal{K}C$\\Hierarchical Clustering based on distributional kernel $\mathcal{K}$}
\label{alg:HKC}
    \textbf{Input}: $D$ - dataset, $s$ - data subset size, $k$ - 
 target number of leaf nodes in the
dendrogram, $\theta$ - hyperparameters of clustering algorithm $\mathcal{A}$
\\
    \textbf{Output}: Dendrogram $T$\\
\begin{algorithmic}[1] %[1] enables line numbers
\vspace{-4mm}
% \STATE \mbox{\small$C=\{G_1,G_2,..G_k\}=Connected\_Components(D_s,E_\tau)$}, where $E_\tau=\{(x,y), \forall x,y \in D_s | x\neq y , \kappa(x,y)>\tau\}$ and $D_s \subset D, |D_s|=s$. 
\STATE Apply $\mathcal{A}(D_s,k,\theta)$ on subset $D_s \subset D, |D_s|=s$, to obtain core clusters \mbox{\small$C=\{G_1,G_2,..G_k\}$}.
\STATE Initialize a binary tree $T$ with root node as $C$.
\STATE Initialize the set of leaf nodes $\mathds{C}=\{C\}$.
\WHILE{$|\mathds{C}|< k$}
\FOR{every leaf node $\hat{C}$ in $\mathds{C}$}
\IF{the number of core clusters in $\hat{C}$ is more than $1$}
\STATE \textcolor{blue}{Split $\hat{C}$ into subclusters $C_{1}$ and $
 C_{2}$}:\\
 $C_j=\{G \in \hat{C} \mid \displaystyle \argmax_{i \in [1,2]} \mathcal{K}(\mathcal{P}_G, \mathcal{P}_{\hat{G}_i}) = j \},$\\$\forall_{j\in [1,2]}$, where $\hat{G}_1$ and $\hat{G}_2$ are the two largest core clusters in $\hat{C}$.
 %$C_j=\{G \in \hat{C}\setminus\{\hat{G}_1,\hat{G}_2\} \mid$ \\$ \displaystyle \argmax_{i \in [1,2]} \mathcal{K}(\mathcal{P}_G, \mathcal{P}_{\hat{G}_i}) = j \}\cup \{\hat{G}_j\},\forall_{j\in [1,2]}$, where $\hat{G}_1$, $\hat{G}_2$ are the largest core clusters in $\hat{C}$.
 \STATE \textcolor{blue}{Update $\mathds{C}$ by replacing $\hat{C}$ with $C_{1}$ and $C_{2}$ in $\mathds{C}$:} $\mathds{C}=\mathds{C} \cup  C_{1} \cup C_{2} \setminus \hat{C} $.
 \STATE \textcolor{blue}{Expand $T$}: adding two child nodes $C_{1},C_{2}$ to $\hat{C}$.
\ENDIF
\ENDFOR
\ENDWHILE
\STATE $G'_j = \{x\in D \mid \displaystyle \argmax_{i \in [1,k]} \mathcal{K}(\delta (x), \mathcal{P}_{G_i}) =j\},\forall_{j\in [1,k]}$
\STATE $\Delta=\lfloor |D|*0.01 \rfloor$, $A_j=G_j,\forall_{j\in [1,k]}$
\FOR{$t=1:100$ AND \mbox{\small$\sum_j|\{x|x\in G'_j, x\notin A_j\}|\ge \Delta$}}
\STATE $A_j=G'_j,\ \forall_{j\in [1,k]}$
\STATE \mbox{\small $G'_j=\{x\in D \mid \mathop{\argmax}\limits_{i \in [1,k]} \mathcal{K}(\delta (x), \mathcal{P}_{A_i}) =j\},$ $\forall_{j\in [1,k]}$}
\ENDFOR  \hfill $\triangleright$ \textcolor{blue}{Refine point assignment}
\STATE for each core cluster $G_j$ in each node of $T$, $G_j\leftarrow G'_j$. \\$\triangleright$ \textcolor{blue}{This step ensures that $T$ has all the points in $D$, not just the core clusters obtained from $D_s$ in line 1. 
%this step is performed within each node of the dendrogram $T$, including the set of leaf nodes $\mathds{C}$.}
}
\STATE \textbf{Return} $T$.
 \end{algorithmic}
\end{algorithm}
% \begin{algorithm}[htb]
% \caption{Example algorithm}
% \label{alg:algorithm}
% \textbf{Input}: Your algorithm's input\\
% \textbf{Parameter}: Optional list of parameters\\
% \textbf{Output}: Your algorithm's output
% \begin{algorithmic}[1] %[1] enables line numbers
% \STATE Let $t=0$.
% \WHILE{condition}
% \STATE Do some action.
% \IF {conditional}
% \STATE Perform task A.
% \ELSE
% \STATE Perform task B.
% \ENDIF
% \ENDWHILE
% \STATE \textbf{return} solution
% \end{algorithmic}
% \end{algorithm}
The proposed Hierarchical Clustering based on distributional kernel $\mathcal{K}$ (H-$\mathcal{K}C$) produces a dendrogram $T$ to achieve the objective of maximizing the total similarity of all clusters ($TSC$):
\begin{equation}
    TSC(T) = \sum_{C \in L_T} \sum_{x \in C} \mathcal{K}(\delta(x),\mathcal{P}_C ),
    \label{equ: local obj}
\end{equation}
where $L_T$ represents the set of $k$ leaf nodes of $T$, and $\delta$ denotes a Dirac measure converting a point into a distribution.

The key steps in Algorithm \ref{alg:HKC} are summarized as follows:
\begin{enumerate}
    \item [*] \textbf{Create a set $C$ of $k$ core clusters from $D_s$} (line 1)\\Each element in $D_s$ is sampled from $D$. An existing clustering algorithm $\mathcal{A}$ can be used here, provided it can find clusters of arbitrary shapes, varied sizes and densities. We show that either psKC \cite{pskc} or DBSCAN \cite{DBSCAN_1996} which employs an appropriate kernel can accomplish this task satisfactorily in Appendix \ref{append: ablation}. 
    \item [*] \textbf{Build a dendrgram $T$ from $C$} (lines 4-12).\\Build a dendrgram from the set of core clusters via bisecting until there are $k$ leaf nodes. Each bisecting selects the largest two core clusters\footnote{Using the two largest core clusters as the basis for each split is an intuitive heuristic which assumes that large clusters are more representative than small clusters.} and assigns each other core cluster to one of them based on distributional kernel $\mathcal{K}$. Note that this process work in a greedy way (unlike existing objective-based DHC methods which exactly optimize a bisecting assessment criterion to build a dendrogram).
    \item [*] \textbf{Assign all points in $D$ to core clusters} (line 13).\\ Assign each unassigned point in $D$ to the most similar core cluster $G_j$ to get $G'_j$.
    \item [*] \textbf{Refine point assignment} (lines 14-18).\\Refine the assignment of points to improve the total similarity of all clusters ($TSC$), defined in Equation \ref{equ: local obj}.
    \item [*] \textbf{Finalize nodes in $T$} (line 19).\\Replace core clusters $G_j$ in each node of $T$ with $G'_j$ to ensure that $T$ has all the points in $D$. 
\end{enumerate}

Note that the objectives achieved before and after the refine point assignment step (lines 14-18) have small or no difference. In other words, the refinement often provides tweaks at the edges of clusters in order to find any minor improvement of the objective obtained before.

H-$\mathcal{K}C$ produces  dendrogram $T$ with the above-mentioned three desired properties:
\begin{enumerate}
    \item Once the $k$ core clusters $G_i$ have been identified in line~1 (by algorithm $\mathcal{A}$ which discovers clusters of arbitrary shapes, varied densities and sizes), they remain the same until they are in the $k$ leaf nodes of $T$ (line 12). Every core cluster is not split in any internal node of the dendrogram. So $T$ has the first and third desired properties.
    \item Similar clusters are grouped into the same subset, when split at each internal node of $T$ based on their similarity as measured by the distributional kernel (line 7 in Algorithm \ref{alg:HKC}). Therefore, $T$ has the second desired property.  
    %No such assurance is provided in the current methods of objective-based DHC.
\end{enumerate}
% No dendrograms provided in the current methods of objective-based DHC can be assured to have these three properties because they employ the point-oriented $h$ criterion to determine each split.
% \textcolor{blue}{No dendrograms provided in the current methods of objective-based DHC can be assured to have these three properties because they employ set-oriented criterias that fail to ensure similar clusters are grouped in each subset.}

% \textbf{Flat Clustering}. If a dendrogram is not required, then step 2 can be ignored to produce a flat clustering $\mathds{C}$ of $k$ clusters which achieves exactly the same objective function:

% \begin{theorem}
% \label{thm1}
%     Both H-$\mathcal{K}BC$ and F-$\mathcal{K}BC$ optimize the same objective function \ref{equ: local obj}.
% \end{theorem}

% We call the above objective based hierarchical clustering and flat clustering defined based on a distributional kernel $\mathcal{K}$ as H-$\mathcal{K}BC$ and F-$\mathcal{K}BC$, respectively.

\subsection{Dendrogram Created by H-$\mathcal{K}C$ through Local $TSC$ Objective Has a Global Guarantee}

Note that the $TSC$ objective considers leaf nodes only in $T$, and we will refer it as local objective function.

We show that the dendrogram created by H-$\mathcal{K}C$ through the local objective $TSC$ has a global guarantee, i.e., the dendrogram produced by H-$\mathcal{K}C$ has a certain level of quality.

To eliminate the effect of the size of the dataset $D$, we modify the original local objective $TSC$ into \[ \max_{T\in\mathcal{T}_k} TSC_l(T),
\] 
where   $TSC_l(T) = \frac{1}{|D|} \sum_{C \in L_T} \sum_{x \in C} \mathcal{K}(\delta(x),\mathcal{P}_C )$
is the local objective function for dendrogram $T$; $L_T$ is the set of  leaf nodes of $T$; and $\mathcal{T}_k$ is the set of dendrograms with $k$ leaf nodes, i.e., $|L_T|=k$.

Although we only refine the local objective  $TSC_l$ in H-$\mathcal{K}C$, the dendrogram $T$ created by H-$\mathcal{K}C$ from a given dataset $D$ has a lower bound of the following global objective function $TSC_g^p$ for $\mathbb{T}_p$, i.e., the set of sub-dendrograms having at least $p$ up to $k$ leaf nodes of $T$:
\begin{equation*}
    TSC_g^p(T)=\frac{1}{|\mathbb{T}_p|}\sum_{T'\in\mathbb{T}_p}TSC_l(T').
\end{equation*}
%where $\mathbb{T}_p$ is the set containing all subdendgrams $T'$ of $T$ with at least $p$ up to $k$ leaf nodes. 

When using distributional kernel $\mathcal{K}$  in H-$\mathcal{K}C$ to produce a fixed $T$ of $k$ leaf nodes, and let $T^q$ be a sub-dendrogram of $T$ with $q  \le k$ leaf nodes,  $TSC_l(T^q)$ is a monotonically non-increasing function wrt $q$, as shown in Lemma \ref{lem1}.
\begin{lemma}\label{lem1}
For a sub-dendrogram $T^q$, when contracting any of its two leaf nodes into one to create a new sub-dendrogram $T^{q-1}$, it holds that
\[
TSC_l(T^{q-1})\geq TSC_l(T^q)-\alpha,
\]
under the assumption that the Euclidean distance between $\mathcal{K}$'s feature maps of the data distributions in two leaf nodes is less than $\alpha$, where $\alpha \ll 1$ is a real constant.
\end{lemma}

\begin{proof}
  For an internal node of $T^q$, denote the number of points in its left (right) leaf node as $n_1=|D_1|$ $(n_2=|D_2|)$, then a contraction of the two leaf nodes into one parent node will decrease $TSC_l(T^q)$ by
  {\small
    \begin{equation*}
        \begin{aligned}
            A = \frac{1}{|D|}[&\sum_{x\in D_1}\mathcal{K}(\delta(x),P_1)+\sum_{x\in D_2}\mathcal{K}(\delta(x),P_2)\\ & -\sum_{x\in D_1}\mathcal{K}(\delta(x),
        P')-\sum_{x\in D_2}\mathcal{K}(\delta(x),P')],
        \end{aligned}
    \end{equation*}
    }\noindent where $P_1$ $(P_2)$ is the left (right) leaf node's data distribution having data subset $D_1$ ($D_2)$. $P'$ is the combination of $P_1$ and $P_2$, having data subset $D_1\cup D_2$. And $\Phi(P')=\frac{n_1 \Phi(P_1)+n_2 \Phi(P_2)}{n_1+n_2}$.
Then 
    {
    \small
    \begin{equation*}
        \begin{aligned}
            % |A| = & |\frac{1}{|D|}[\frac{n_2}{n_1+n_2}\sum_{x\in D_1}\mathcal{K}(\delta(x),P_1-P_2)\\
            % & +\frac{n_1}{n_1+n_2}\sum_{x\in D_2}\mathcal{K}(\delta(x),P_2-P_1)]|\\
           |A| = & |\frac{1}{|D|}[\frac{n_2}{n_1+n_2}\sum_{x\in D_1}\left \langle\Phi(\delta(x)),\Phi(P_1)-\Phi(P_2)\right\rangle\\
            & +\frac{n_1}{n_1+n_2}\sum_{x\in D_2}\left \langle\Phi(\delta(x)),\Phi(P_2)-\Phi
            (P_1)\right\rangle]|\\
         \leq & \frac{n_2|D_1|+n_1|D_2|}{(n_1+n_2)|D|}\|\Phi(P_1)-\Phi(P_2)\|\\
         \leq & \alpha. 
        \end{aligned}
    \end{equation*} 
    }
    
The second line holds because of Cauchy-Schwarz inequality and the norm of $\mathcal{K}$'s feature map is less than 1. The last line holds based on the assumption stated in the Lemma. Finally, we have $TSC_l(T^{q-1})\geq TSC_l(T^q)-\alpha$.     
\end{proof}

% Although we only optimize the local objective function, the dendrogram created by the procedure will give a lower bound of the following global objective function
% \begin{equation}
%     F_g^p(T,X)=\frac{1}{|\mathbb{T}_p|}\sum_{T'\in\mathbb{T}_p}F_l(T',X),
% \end{equation}
% where $\mathbb{T}_p$ is the set containing all subdendgrams $T'$ of $T$ with at least $p$ up to $k$ leaf nodes.

\noindent With Lemma \ref{lem1}, we can get the following theorem.
\begin{theorem}\label{thm2}
Under the same assumption in Lemma \ref{lem1}, it holds for a dendrogram $T$ with $k$ leaf nodes created by H-$\mathcal{K}C$ that:
    \begin{equation*}
       \max_{T'\in\mathcal{T}_k} TSC_l(T')-(k-p)\alpha\leq TSC_g^p(T).
    \end{equation*}
\end{theorem}

\begin{proof}
    A dendrogram with $k$ leaf nodes can contract $k-q$ times to form a sub-dendrogram $T^q$ 
 with $q$ leaf nodes. According to Lemma \ref{lem1}, each contraction operation induces a decrease from $\max_{T'\in\mathcal{T}_k} TSC_l(T')$ at most $\alpha$. Then we have $\max_{T'\in\mathcal{T}_k} TSC_l(T')-(k-q)\alpha\leq TSC_l(T^q)$ for a sub-dendrogram $T^q$ with $q$ leaf nodes of $T$. By averaging over $q$ (from $p$ to $k$), we have the stated result in Theorem~\ref{thm2}.
    %$\max_{T'\in\mathcal{T}_k} TSC_l(T')-(k-p)\alpha\leq TSC_g^p(T)$.
\end{proof}

Compared with SpecWRSC (Bisect-Kmeans), which proves that its dendrogram approximates the optimal Dasgupta cost function (Revenue Objective), the dendrogram produced by H-$\mathcal{K}C$ provides a lower bound of $TSC_g^p$.
\begin{figure}[htb]
\begin{subfigure}
        {.245\linewidth}
      \centering   
\includegraphics[scale=0.24]{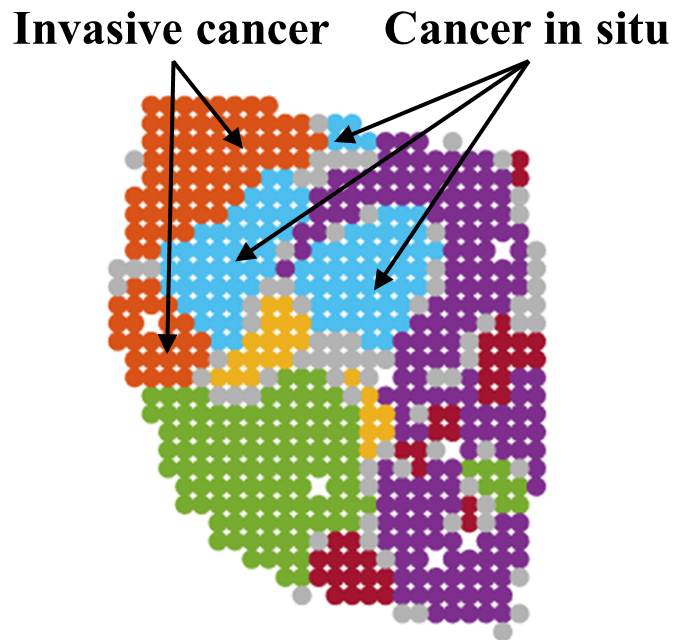}
        \caption{Ground Truth}
        \label{subfig: HER2 label}
    \end{subfigure}
\begin{subfigure}
        {.245\linewidth}
      \centering   
\includegraphics[width=\linewidth]{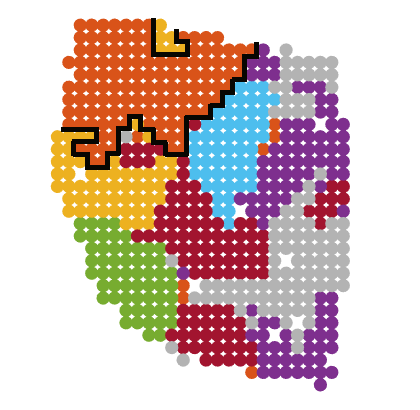}
        \caption{SpecWRSC}
        \label{subfig: flat HER2 SpecWRSC}
\end{subfigure}
\begin{subfigure}
        {.245\linewidth}
      \centering   
\includegraphics[width=\linewidth]{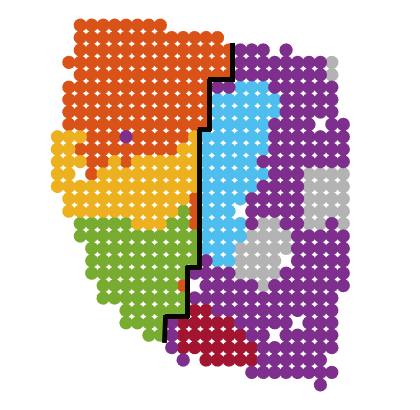}
        \caption{Bisect-Kmeans}
        \label{subfig: flat HER2 Bisect-Kmeans}
\end{subfigure}
\begin{subfigure}
        {.245\linewidth}
      \centering   
\includegraphics[width=0.93\linewidth]{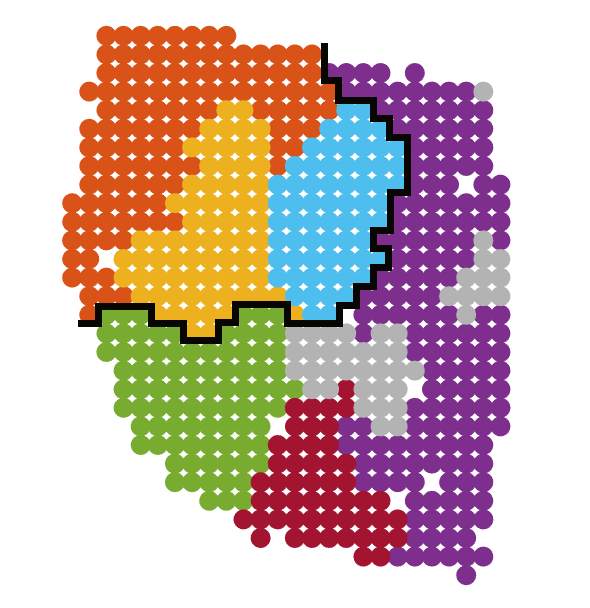}
\caption{H-$\mathcal{K}C$}
        \label{subfig: flat HER2 HKBC}
    \end{subfigure} 
     \begin{subfigure}
        {.245\linewidth}
              \centering   
\includegraphics[width=\linewidth]{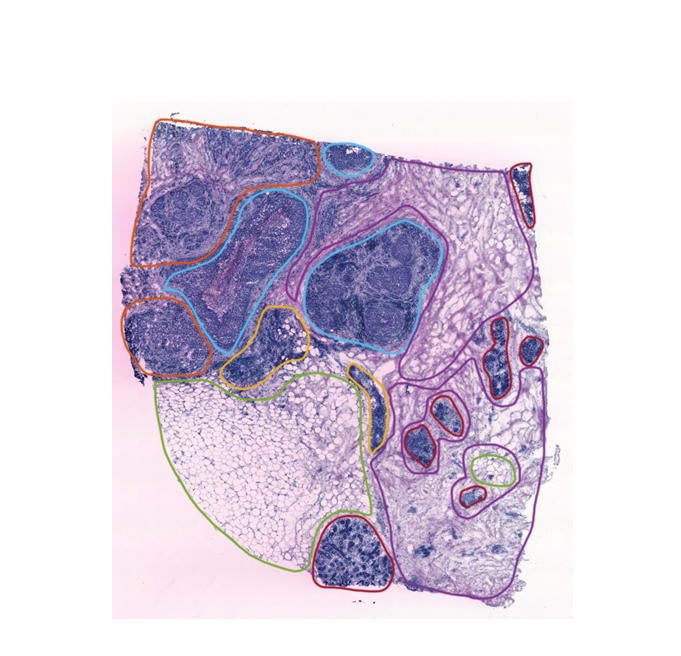}
        \caption{HER2 tumor data}
        \label{subfig: flat HER2 SpaGCN}
         \end{subfigure}
\begin{subfigure}
        {.245\linewidth}
      \centering   
\includegraphics[width=1.05\linewidth]{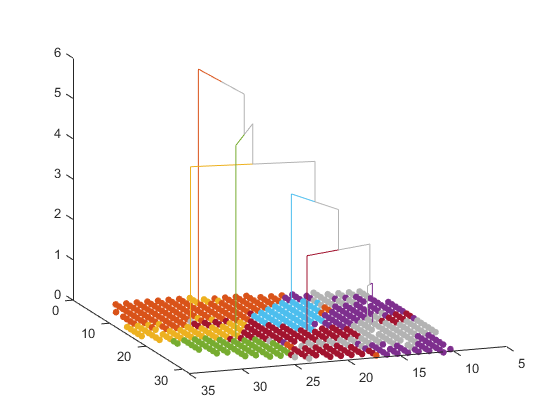}
        \caption{SpecWRSC: $\wp=0.58$}
        \label{subfig: dendro HER2 SpecWRSC}
\end{subfigure} 
\begin{subfigure}
        {.245\linewidth}
      \centering   
\includegraphics[width=1.05\linewidth]{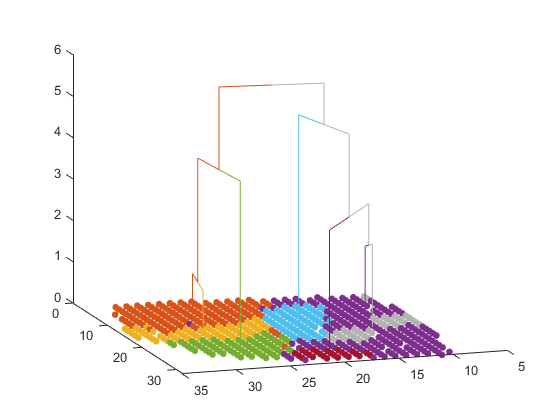}
        \caption{Bisect-Kmeans: $\wp=0.55$}
        \label{subfig: dendro HER2 Bisect-Kmeans}
\end{subfigure}
\begin{subfigure}
        {.245\linewidth}
      \centering   
\includegraphics[width=1.05\linewidth]{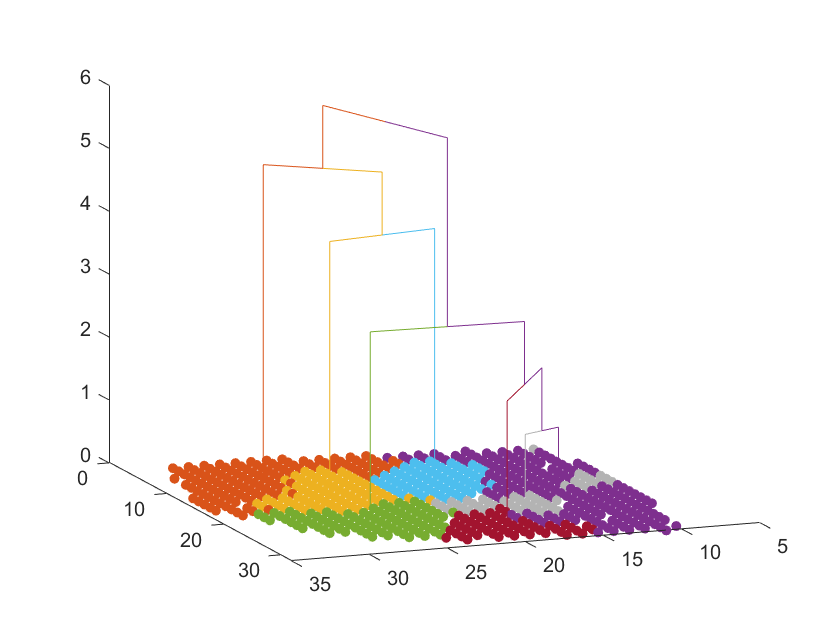}
        \caption{H-$\mathcal{K}C$: $\wp=0.75$}
        \label{subfig: dendro HER2 HKBC}
    \end{subfigure} 
    \caption{HER2 tumor dataset \protect\cite{andersson2020spatial} with ground-truth labels (a) and manual annotation plot of the tissue sample (e). Dendrograms of SpecWRSC (f), Bisect-Kmeans (g) and H-$\mathcal{K}C$ (h), and their flat clustering results shown in the leaf nodes are given in (b), (c) and (d), respectively. The solid black line  indicates the result of the first split in (b)-(d). The   Dendrogram Purity ($\wp$) is calculated in terms of the `Invasive cancer' and `Cancer in situ' regions. The flat clustering result of SpaGCN \protect\cite{hu2021spagcn}, a recent end-to-end deep learning flat clustering method which produces a comparable flat clustering of Bisect-Kmeans (c), is provided in Appendix \ref{append: SpaGCN}.
}
\label{fig: HER2}
\end{figure}

\subsection{H-$\mathcal{K}C$ could employ AHC}
\begin{proposition}
\label{prop: ahc}
With the same set $C$ of core clusters, AHC creates a dendrogram $T'$ similar to $T$ created by H-$\mathcal{K}C$ at line 12, if it merges two nodes $X=\{C_{i_1},...,C_{i_m}\}$ and $Y=\{C_{j_1},...,C_{j_n}\}$ with the maximum $f(X,Y):=\max_{C_i\in X,C_j\in Y}\mathcal{K}(C_i,C_j)$.
\end{proposition}
Note that $f$ is a single-linkage function that uses $\mathcal{K}$.
Although H-$\mathcal{K}C$ does not minimize $f$ in each split, it has an upper bound of $f$ for each split. Therefore, it yields a dendrogram similar to that produced by AHC using the same $f$,  where the arguments in $f$ are clusters/distributions and their similarity is measured by a distributional kernel. The proof is provided in Appendix \ref{append: proof}.

Proposition \ref{prop: ahc} is consistent with the fact \cite{LABBE2023555,gagolewski2024clustering} that AHC and DHC create the same dendrogram if they use the same single-linkage function to assess the goodness of a merge and a split, respectively. In other words, lines 4-12 in H-$\mathcal{K}C$ could use AHC because $f$ is a single-linkage function in terms of distributional kernel; but Bisect-Kmeans and SpecWRSC cannot employ AHC because their objective functions are equivalent to the weighted average linkage function.

The stated procedure (lines 4-12) in Algorithm \ref{alg:HKC} is preferred over the use of AHC because the former is more efficient to build a dendrogram, i.e., %despite the choice of $h$, $h(X,Y)=\max_{C_i\in X,C_j\in Y}\mathcal{K}(C_i,C_j)$, $\min_{C_i\in X,C_j\in Y}\mathcal{K}(C_i,C_j)$ or $\frac{1}{|X||Y|}\sum_{C_i\in X,C_j\in Y}\mathcal{K}(C_i,C_j)$, 
for each split using $f(X,Y)$,  where $|X|=|Y|=m$, the time complexity is $O(m)$ for H-$\mathcal{K}C$; but it is $O(m^2)$ for merge using $f$ in AHC.

% \textcolor{red}{Need to clear what you mean by  `produces the same dendrogram' in the first paragraph; and `produces a similar dendrogram'}

% \textbf{Flat Clustering}. If a dendrogram is not required, then step 3-12 in Algorithm \ref{alg:HKC} can be ignored to produce a flat clustering $\mathds{C}$ of $k$ clusters $\{G'_1,G'_2,\dots,G'_k\}$. We refer this flat counterpart of H-$\mathcal{K}BC$ as F-$\mathcal{K}BC$, and they achieves exactly the same $TSC$ objective function:

% \begin{theorem}
% \label{thm1}
%     Both H-$\mathcal{K}BC$ and F-$\mathcal{K}BC$ optimize TSC.
% \end{theorem}
% The proof is provided in the Appendix.

\section{Empirical Evaluation}
\begin{figure}[htb]
\begin{subfigure}
        {.33\linewidth}
      \centering   
\includegraphics[scale=0.17]{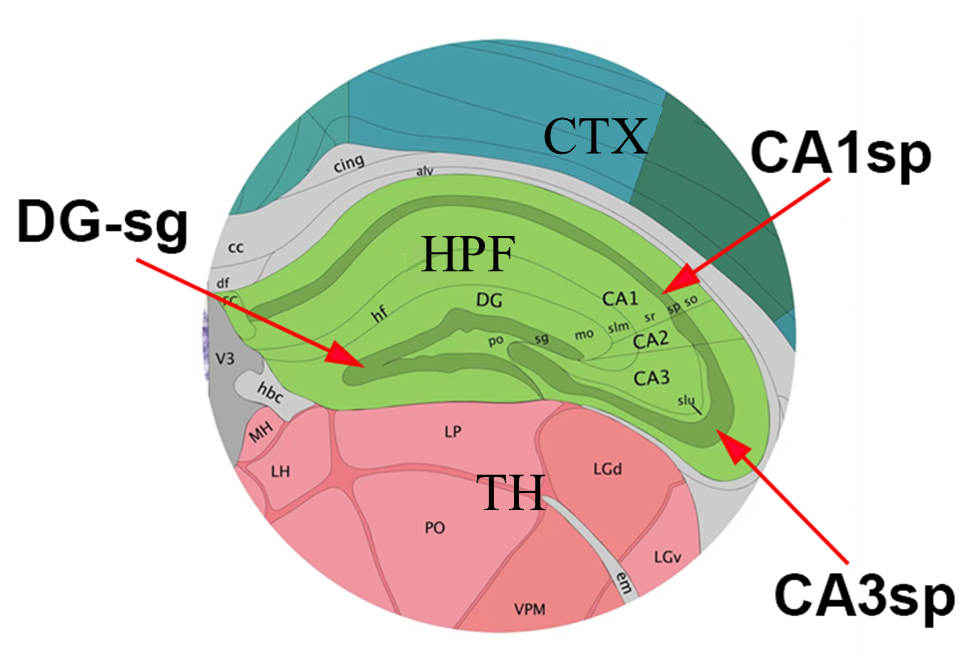}
        \caption{ Allen Brain Atlas}
        \label{subfig: Slide-seq}
    \end{subfigure}
     \begin{subfigure}
        {.33\linewidth}
      % \centering   
\includegraphics[scale=0.19]{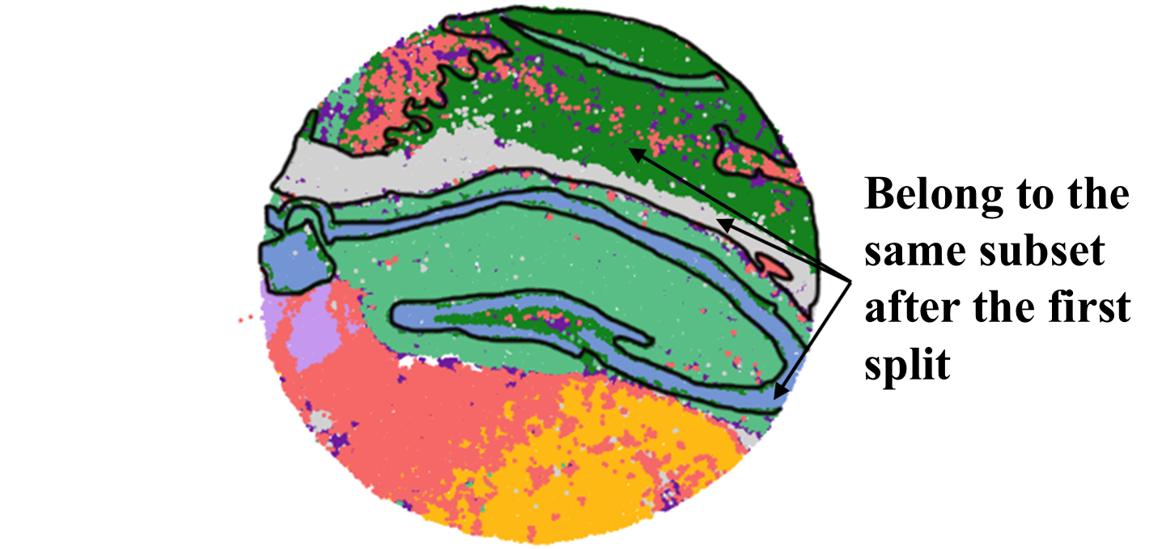}
        \caption{Bisect-Kmeans}
        \label{subfig: flat slide Bisect-Kmeans}
    \end{subfigure} 
   \begin{subfigure}
        {.33\linewidth}
      \centering   
\includegraphics[scale=0.205]{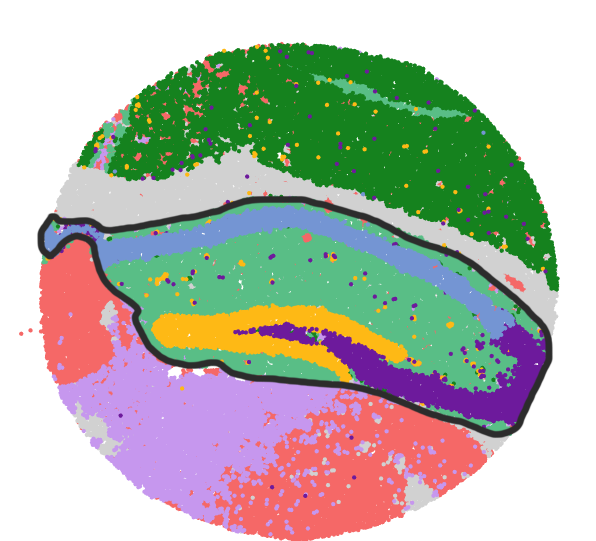}
\caption{H-$\mathcal{K}C$}
        \label{subfig: flat slide HKBC}
    \end{subfigure} 
     \begin{subfigure}
        {.49\linewidth}
      \centering   
\includegraphics[scale=0.2]{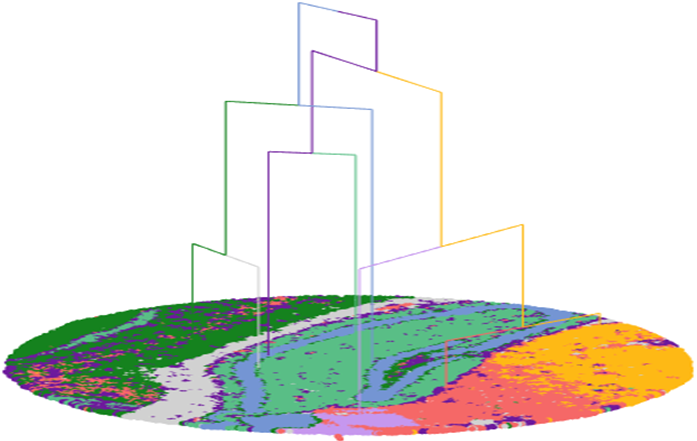}
        \caption{Bisect-Kmeans}
        \label{subfig: dendro slide Bisect-Kmeans}
    \end{subfigure}
   \begin{subfigure}
        {.49\linewidth}
      \centering   
\includegraphics[scale=0.2]{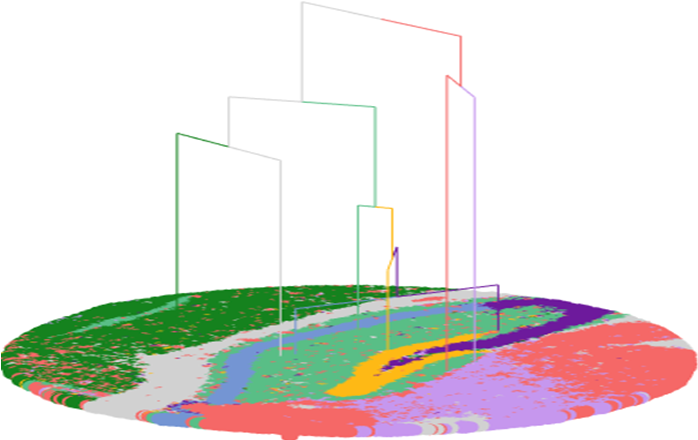}
        \caption{H-$\mathcal{K}C$}
        \label{subfig: dendro slide HKBC}
    \end{subfigure} 
    \caption{The Slide-seq V2 dataset on mouse hippocampus \protect\cite{stickels2021highly}. Allen Brain Atlas (a) shows the three primary regional divisions: CTX, HPF and TH. The coronal mouse olfactory bulb from the Allen Brain Atlas in the central area is labelled as CA1sp and CA3sp. Dendrograms of Bisect-Kmeans (d) and H-$\mathcal{K}C$ (e), and their flat clustering results shown in the
leaf nodes are given in (b) and (c), respectively. The solid black line  indicates the result of the first split in (b) and the third split in (c).  
% \textcolor{red}{Replace `one subset after the first split' with `Belong to the same subset after the first split'}
}
\label{fig: Slide}
\end{figure}

We evaluate the effectiveness of H-$\mathcal{K}C$ in two biological datasets derived from Spatial Transcriptomics \cite{marx2021method}, which is chosen to be the method of the year in 2021 by Nature Methods. We select these datasets because DHC can potentially provide a deeper insight into the structure of the tissue samples. The first one is the over-expression of HER2 (human epidermal growth factor receptor 2) on tumor cells which identifies two major subtypes of breast cancer. 
%We have selected patient H section 1 as our example for analysis. 
It includes 10053 genes and 607 cells collected through spatial transcriptomics. The second one is the Slide-seq V2 dataset on mouse hippocampus which contains 53208 cells, each having 23264 genes. Note that Slide-seq V2  does not have ground-truth\footnote{This dataset has no ground truth. The Allen Brain Atlas (Figure \ref{subfig: Slide-seq}), which provides the regional divisions, is employed to understand the original tissue structure.} label. Each dataset is preprocessed, which integrates the spatial location information of individual cells in a sample tissue and the gene expression information of each cell, into vector representation (HER2 has 120 attributes and 607 points; and Slide-seq V2 has 80 attributes and 51367 points). The details are given in Appendix \ref{append: ST data}.

%Dendrogram purity \cite{heller2005bayesian,kobren2017hierarchical} is often utilized to evaluate how well the leaf node clusters of a dendrogram match the ground-truth clusters, without considering dendrogram's structure. 
As stated in the Related Work section, since there are no existing metrics to evaluate the quality of a dendrogram's structure, we use visualization to compare the structures of dendrograms generated by various DHC algorithms.

For the HER2 tumor dataset, the two cancer regions are the center of attention, i.e., `Invasive cancer' and `Cancer in situ', shown in Figure \ref{subfig: HER2 label}.  The observations from the dendrograms produced by Bisect-Kmeans, SpecWRSC and  H-$\mathcal{K}C$,  shown in Figure \ref{fig: HER2}, are given as follows:
\begin{itemize}
    \item SpecWRSC produces a highly unbalanced dendrogram which divides a cluster from the rest of the clusters at each internal node. Part of the `Cancer in situ' is mixed with the `Invasive cancer'.
    % Only parts of the two important regions, i.e., `Invasive cancer' and `Cancer in situ', are correctly identified.
    \item Although both Bisect-Kmeans and H-$\mathcal{K}C$ produce rather balanced dendrograms, H-$\mathcal{K}C$ has grouped the two cancer regions into the same subset, and the non-cancer regions (regular tissue) into another in the first split; but Bisect-Kmeans divides the two cancer regions into two separate subsets, mixing the cancer and non-cancer regions. In addition, Bisect-Kmeans incorrectly combines `Cancer in situ' with `Invasive cancer', treating them as a single cluster, as SpecWRSC has done. Only H-$\mathcal{K}C$ largely correctly identified `Invasive cancer' (into one cluster) and `Cancer in situ' (into two regions). That is the reason why H-$\mathcal{K}C$ has significantly higher Dendrogram Purity ($\wp$) than SpecWRSC and Bisect-Kmeans.
\end{itemize}

For Slide-seq V2, the center of attention is the three regions marked as  CA1sp, CA3sp and DG-sg  in Figure~\ref{subfig: Slide-seq}:
\begin{itemize}
    \item Bisect-Kmeans produces a dendrogram with poor structure: its first split incorrectly groups CA1sp, DG-sg CA3sp, and CTX into the same subset, as shown in (b). Note that this subset consists of fragment regions separated by regions belonging to anohter subset. Further, Bisect-Kmeans  groups CA1sp, DG-sg and CA3sp at the center of attention into one single cluster in a leaf node.
    \item H-$\mathcal{K}C$ produces a dendrogram that is consistent with the Allen Brain Atlas. The first split separates TH (the bottom region) from CTX \& HPF (as a subset). The third split separates HPF from CTX. And the subsequent splits successfully  CA1sp, CA3sp ad DG-sg as three different clusters.
    \item SpecWRSC  produces out-of-memory  error on this large dataset; thus it has no result in Figure \ref{fig: Slide}.
\end{itemize}

In summary, both existing objective-based DHC, i.e., SpecWRSC and Bisect-Kmeans, do not produce dendrograms which have the three above-mentioned desired properties. Only H-$\mathcal{K}C$ is capable of generating a dendrogram with a structure that is consistent with the biological regions and accurately detects clusters in the center of attention.

\subsection{Comparison with More Baselines} 

In addition to SpecWRSC and Bisect-Kmeans, we include comparison with more baselines (DHC methods DIANA \cite{DIANA-1990}, HDP \cite{zhu2022hierarchical} and AHC method SCC \cite{scc}) on additional 13 datasets. The results in terms of dendrogram purity are shown in Table \ref{tab: baseline} of Appendix \ref{append: more baselines}.
\begin{figure}[h]
      \centering   
\includegraphics[scale=0.6]{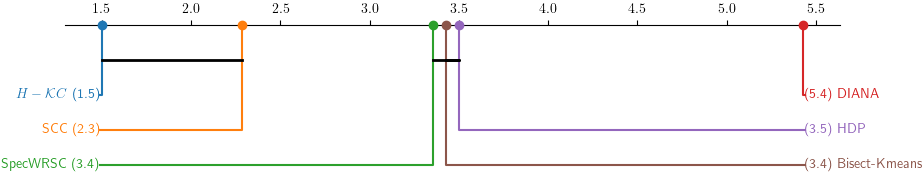}
        \caption{Critical difference diagram.}
        \label{fig: cd diagram}
\end{figure}
Here we summarize the results in Appendix \ref{append: more baselines} with the corresponding Critical Difference (CD) diagram of the post-hoc Nemenyi test in Figure \ref{fig: cd diagram}. H-$\mathcal{K}C$ has the highest average ranking and there is no critical difference between H-$\mathcal{K}C$ and SCC. This is consistent with our analysis in Proposition \ref{prop: ahc}.

\section{Discussion}
The time complexity of H-$\mathcal{K}C$ is $\mathcal{O}(kn+s^2)$, where $n$ is the given dataset size, $s$ is the data subset size used to produce the dendrogram (at line 12 in Algorithm \ref{alg:HKC}), and $k$ is the  number of leaf nodes. Both Bisect-Kmeans and H-$\mathcal{K}C$ have linear time complexity. The time complexity of SpecWRSC is at least quadratic. The detailed analysis and the scaleup test are given in Appendix \ref{append: scale}. 

For the choice of distributional kernel in H-$\mathcal{K}C$, we use the recently introduced Isolation Distributional Kernel (IDK) \cite{ting2020isolation} which 
%origins from Isolation Kernel \cite{ting2018isolation}, which 
can deal with clusters of varied densities better than Gaussian Distributional Kernel (GDK). Our ablation study comparing IDK and GDK in H-$\mathcal{K}C$ is consistent with the previous findings. We have used psKC as $\mathcal{A}$ in line 1 of H-$\mathcal{K}C$ to produce the results in the last section. The examination of different existing algorithms as $\mathcal{A}$, three ablation studies and a hyperparameter sensitivity analysis of H-$\mathcal{K}C$ are provided in Appendix \ref{append: more exps}.

\section{Concluding Remarks}

Our discovery of the new approach to objective-based DHC has a much wider implication. Exact optimization on the set-oriented bisecting assessment criterion has been the tool of choice in order to produce an optimized dendrogram. Yet, we show that a greedy algorithm H-$\mathcal{K}C$ exists to lower bound an explicit objective,  producing a dendrogram with the three desired properties that existing optimization-based methods could not produce, and H-$\mathcal{K}C$ runs in linear time.

This is possible mainly because of a paradigm shift---the problem is defined in terms of distributions rather than the traditional set-oriented approach. Two key differences are: (I) Bisecting at each internal node of a dendrogram can be better achieved via a distributional kernel, without a typical set-oriented bisecting assessment criterion. The exact optimization of a set-based cost function has no payoff mainly because it does not ensure that similar clusters are grouped in the same subset in each split. (II) The first step in the distribution-oriented approach in deriving a set of core clusters is crucial in (a) discovering clusters of varying shapes, sizes and densities, and (b) enabling a dendrogram with the three desired properties to be constructed in a new way as stated in (I). We show that this step can be accomplished efficiently from a small subset of a given dataset. 

% \section*{Impact Statement}
% This work proposes an objective-based divisive hierarchical clustering method from the distributional perspective that aims to produce a dendrogram with the guaranteed structural fidelity that groups similar clusters into a same subset and avoids unwarranted splitting. This dendrogram structure depicts trust-worthy and realistic cluster relationships in an exploratory analysis, especially in scientific domains where hierarchy matters (e.g., spatial transcriptomics). Dendrograms produced by existing methods are often not sensible, e.g., having dissimilar clusters grouped into a same subset in the hierarchy.

%%%%%%%%%%%%%%%%%%%%%%%%%%%%%%%%%%%%%%%%%%%%%%%%%%%%%%%%%%%%%%%%%%%%%%%%%%%%%%%
%%%%%%%%%%%%%%%%%%%%%%%%%%%%%%%%%%%%%%%%%%%%%%%%%%%%%%%%%%%%%%%%%%%%%%%%%%%%%%%
% APPENDIX
%%%%%%%%%%%%%%%%%%%%%%%%%%%%%%%%%%%%%%%%%%%%%%%%%%%%%%%%%%%%%%%%%%%%%%%%%%%%%%%
%%%%%%%%%%%%%%%%%%%%%%%%%%%%%%%%%%%%%%%%%%%%%%%%%%%%%%%%%%%%%%%%%%%%%%%%%%%%%%%
\newpage
\appendix
\onecolumn

\section{Background}
\subsection{Isolation Distributional Kernel}

We first describe Isolation Kernel \cite{ting2018isolation}. Let $D \subset \mathbb{R}^d$ be a dataset sampled from an unknown distribution $P_D$; and $\mathbb{H}_\psi(D)$ denote the set of all partitionings $H$ that are admissible from $\mathcal{D} \subset D$, which is a random subset of $\psi$ points. Each partition $\theta[\mathbf{z}] \in H$ isolates a point $\mathbf{z} \in \mathcal{D}$ from the rest of the points in $\mathcal{D}$.
Let $\mathds{1}(\cdot)$ be an indicator function.

 For any two points $x,y \in \mathbb{R}^d$,
	Isolation Kernel of $x$ and $y$ is defined to be
	the expectation taken over the probability distribution on all partitionings $H \in \mathds{H}_\psi(D)$ that both $x$ and $y$  fall into the same isolating partition $\theta[\mathbf{z}] \in H$, where $\mathbf{z} \in \mathcal{D} \subset D$, $\psi=|\mathcal{D}|$:
\begin{equation*}
\begin{aligned}
    \kappa_I(x,y\ |\ D) & = {\mathbb E}_{\mathds{H}_\psi(D)} [\mathds{1}(x,y \in \theta[\mathbf{z}]\ | \ \theta[\mathbf{z}] \in H)] \\
     & = {\mathbb E}_{\mathcal{D} \subset D} [\mathds{1}(x,y\in \theta[\mathbf{z}]\ | \ \mathbf{z}\in \mathcal{D})] 
\end{aligned}
\label{eqn_kernel}
\end{equation*}

%\noindent
%where $\mathds{1}(\cdot)$ is an indicator function.

In practice, Isolation Kernel $\kappa_I$ is constructed using a finite number of partitionings $H_i, i=1,\dots,t$, where each $H_i$ is created using randomly subsampled $\mathcal{D}_i \subset D$; and $\theta$ is a shorthand for $\theta[\mathbf{z}]$:
\begin{equation*}
    \begin{aligned}
        \kappa_I(x,y\ |\ D)   &=   \frac{1}{t} \sum_{i=1}^t   \mathds{1}(x,y \in \theta\ | \ \theta \in H_i) \\
&= \frac{1}{t} \sum_{i=1}^t \sum_{\theta \in H_i}   \mathds{1}(x\in \theta)\mathds{1}(y\in \theta) 
    \end{aligned}
    \label{Eqn_IK}
\end{equation*}
Let Isolation Kernel be implemented using isolating hyperspheres \cite{ting2020isolation} for each partitioning from a sample $\mathcal{D}$ of $\psi$ points. The radius of each hypersphere centered at $\mathbf{z}$ is the distance between $\mathbf{z}$ and its nearest neighbor in $\mathcal{D}\setminus \{\mathbf{z}\}$.
Given a partitioning $H_i$, let feature $\phi_i(x)$ be a $\psi$-dimensional binary column vector representing all hyperspheres $\theta_j \in H_i$, $j=1,\dots,\psi$; where $x$ falls into one of the $\psi$ hyperspheres or none.
The $j$-component of the vector due to $H_i$ is:
$\phi_{ij}(x)=\mathds{1}(x\in \theta_j\ |\ \theta_j\in H_i)$. Given $t$ partitionings, $\phi(x)$ is the concatenation of $\phi_1(x),\dots,\phi_t(x)$. After that, $\phi(x)$ is divided by $\sqrt{t}$ for normalization such that $\|\phi(x)\|=1$.

Isolation Distributional Kernel \cite{ting2020isolation} is developed from Isolation Kernel $\kappa_I$. Since $\kappa_I$ has finite dimensional feature map $\phi$. The feature map of IDK $\Phi$ is simply calculated as $\Phi(\mathcal{P}_D)=\frac{1}{|D|}\sum_{x\in D}\phi(x)$. Since $\|\phi(x)\|=1$, it holds that $\|\Phi(\mathcal{P}_D)\|\leq 1$.

\subsection{Point-Set Kernel Clustering}
Point-Set kernel clustering \cite{pskc} employs the point-set kernel (distribution kernel where one distribution is a single point) $K$ to characterize clusters. It identifies all members of each cluster by first locating the seed of the dataset. Then, it expands its members in the cluster's local neighborhood 
which grows at a set rate ($\varrho$) 
incrementally; and it stops growing when all unassigned points have similarity w.r.t. the cluster falling below a threshold ($\tau$). 
The process repeats for the next cluster using the remaining points in the given dataset $D$, yet to be assigned to any clusters found so far, until $D$ is empty or no point can be found which has similarity more than $\tau$. All remaining points after the clustering process are noise as they are less than the set threshold for each of the clusters discovered. The  \texttt{psKC} procedure is shown in Algorithm \ref{alg:pset-KC}, where $K$ is point-set kernel.

\begin{algorithm}[h]
 \caption{point-set Kernel Clustering (\texttt{psKC})}
 \label{alg:pset-KC}
 \textbf{Input}: $D$: dataset,  $\tau$: similarity threshold, $\varrho$: growth rate\\
\textbf{Output}: $G^j, j=1,\dots,k$: $k$ clusters, $N$: noise set\\
\begin{algorithmic}[1]
\vspace{-4mm}
\STATE $k=0$\\
\WHILE{$|D|>1$}
\STATE $x_p = \argmax_{x \in D} K(x,D)$ \hfill $\triangleright$ \textcolor{blue}{Seed}\\
\STATE $x_q = \argmax_{x \in D \setminus \{x_p\}} K(x, \{x_p\})$\\
\STATE $\gamma = (1-\varrho) \times K(x_q, \{x_p\})$ \\
\IF{$\gamma \le \tau$}
\STATE Terminate while-do loop\\
\ENDIF
\STATE $k$++\\
\STATE $G_0^k = \{x_p, x_q \}$ \hfill $\triangleright$ \textcolor{blue}{Initial cluster $k$}\\
\FOR{($i=1;\ \gamma  > \tau;\ i$++)}
  \STATE $G_i^k = \{ x \in D\ |\ K(x, G_{i-1}^k) > \gamma \}$\\
  \STATE $\gamma = (1-\varrho) \gamma$\\
  \ENDFOR
  \STATE $G^k=G_{i-1}^k$ \hfill $\triangleright$ \textcolor{blue}{Cluster $k$ grown}\\
  \STATE $D= D \setminus G^k$\\
\ENDWHILE
\STATE $N = D$\;
\STATE \textbf{Return} $G^j, j=1,\dots,k;\ N$\\
 \end{algorithmic}
\end{algorithm}

In Algorithm \ref{alg:pset-KC}, the number of clusters are not automatically determined. In order to fit our setting which takes the number of clusters $k$ as input and incorporate the adaption from point-set kernel to distributional kernel, we introduce distributional kernel based \texttt{psKC} ($\mathcal{K}$-\texttt{psKC}) in Algorithm \ref{alg:pset-KC2}, where $\mathcal{K}$ represents distributional kernel. We apply Algorithm \ref{alg:pset-KC2} on a subset of the entire dataset to obtain $k$ core clusters in the first line of H-$\mathcal{K}C$. The noise set $N$ is not part of the core clusters. 

\begin{algorithm}[h]
 \caption{Distributional kernel based  \texttt{psKC} ($\mathcal{K}-\texttt{psKC}$)}
 \label{alg:pset-KC2}
 \textbf{Input}: $D$: dataset, $k$: number of clusters, $\tau$: similarity threshold, $\varrho$: growth rate\\
\textbf{Output}: $G^j, j=1,\dots,k$: $k$ clusters\\
\begin{algorithmic}[1]
\vspace{-4mm}
\STATE $j=0$\\
\WHILE{$|D|>1$ and $j<k$}
\STATE $x_p = \argmax_{x \in D} \mathcal{K}(\delta(x),P_D)$ \hfill $\triangleright$ \textcolor{blue}{Seed}\\
\STATE $x_q = \argmax_{x \in D \setminus \{x_p\}} \mathcal{K}(\delta(x), \delta(x_p))$\\
\STATE $\gamma = (1-\varrho) \times \mathcal{K}(\delta(x_q), \delta(x_p))$ \\
\IF{$\gamma \le \tau$}
\STATE Terminate while-do loop\\
\ENDIF
\STATE $j$++\\
\STATE $G_0^j = \{x_p, x_q \}$ \hfill $\triangleright$ \textcolor{blue}{Initial cluster $k$}\\
\FOR{($i=1;\ \gamma  > \tau;\ i$++)}
  \STATE $G_i^j = \{ x \in D\ |\ \mathcal{K}(\delta(x), P_{G_{i-1}^j}) > \gamma \}$\\
  \STATE $\gamma = (1-\varrho) \gamma$\\
  \ENDFOR
  \STATE $G^j=G_{i-1}^j$ \hfill $\triangleright$ \textcolor{blue}{Cluster $j$ grown}\\
  \STATE $D= D \setminus G^j$\\
\ENDWHILE
% \STATE $N = D$\;
\STATE \textbf{Return} $\{G^1,G^2,...,G^k\}$.\\
 \end{algorithmic}
\end{algorithm}

\subsection{Dendrogram Purity}\label{appendix: purity}
Given a dendrogram $T$ produced by a hierarchical clustering algorithm from a dataset $D=\{\mathbf{x}_1,\dots,\mathbf{x}_n\}$.
Let $\mathsf{P}=\left\{\left(\mathbf{x}, \mathbf{x}^{\prime}\right) | \mathbf{x} \ne \mathbf{x}^{\prime} \in D, \ell(\mathbf{x})=\ell(\mathbf{x}^{\prime})\right\}$  be
the set of pairs of different points that
have the same ground-truth cluster label, where $\ell(\mathbf{x})$ denotes the ground-truth  cluster label  of $\mathbf{x}$; and $L_i, i=1,\dots,\Bbbk$
be the ground-truth labels of $\Bbbk$ clusters.
Formally, the $Dendrogram$ $Purity$ of $T$ is defined as ~\cite{kobren2017hierarchical}
\begin{equation}~\label{equ:dp}
	\mathrm{Purity}(T)=\frac{1}{\left|\mathsf{P}\right|} \sum_{i=1}^{\Bbbk} \sum_{(\mathbf{x}, \mathbf{x}^{\prime}) \in \mathsf{P}, \ell(\mathbf{x})=L_i} \mathfrak{f}\left(\mathfrak{g}\left(\mathfrak{h}\left(\mathbf{x}, \mathbf{x}^{\prime}\right)\right),L_i\right) \nonumber,
	%\ell(\mathbf{x})= \ell(\mathbf{x}^{\prime}) = C_{t};
\end{equation}
where $\mathfrak{h}\left(\mathbf{x}, \mathbf{x}^{\prime}\right)$ is the least common ancestor of $\mathbf{x}$ and $\mathbf{x}^{\prime}$ in $\mathbb{T}$, $\mathfrak{g}(\mu)\subset D$
is the set of points in all the descendant leaf nodes of internal node $\mu$ in
$T$,
and $\mathfrak{f}(S, L_i)=\frac{|\{ \mathbf{x} \in S\ |\ \ell(\mathbf{x})=L_i\}|}{S}$
computes the fraction of $S$ that
matches the ground-truth label $L_i$. 

\section{Proof}
\label{append: proof}

\begin{proposition}
With the same leaf nodes of dendrogram $T$ created by H-$\mathcal{K}C$, AHC creates a similar dendrogram $T'$, if it merges two nodes $X=\{C_{i_1},...,C_{i_m}\}$ and $Y=\{C_{j_1},...,C_{j_n}\}$ with the maximum $f(X,Y):=\max_{C_i\in X,C_j\in Y}\mathcal{K}(C_i,C_j)$.
\end{proposition}

\begin{proof}

    AHC and DHC create the same dendrogram \cite{LABBE2023555,gagolewski2024clustering} if they use the same single-linkage  function $f$ to assess the goodness of a merge and a split, respectively. That is, if AHC merges the two most similar nodes at each step according to  $f(X,Y)$, which have subsets $X$ and $Y$,  then a DHC algorithm produces the same dendrogram if every bisecting minimizes $f(X,Y)$. 

Let $f(X,Y)=\max_{C_i\in X,C_j\in Y}\mathcal{K}(\mathcal{P}_{C_i},\mathcal{P}_{C_j})$, where $\mathcal{K}$ is a distribution kernel.  The following AHC and DHC algorithms produce the same dendrogram:
%obtained through H-$\mathcal{K}BC$.
\begin{enumerate}
    \item[*] \textit{AHC merges two nodes $X=\{C_{i_1},...,C_{i_m}\}$ and $Y=\{C_{j_1},...,C_{j_n}\}$ with the maximum $f(X,Y)$.}
    \item[*] \textit{DHC minimizes $f(X,Y)$ when it splits a node into two child nodes $X,Y$.}
\end{enumerate}

Note that the bisecting criterion of H-$\mathcal{K}C$, i.e., $\{G \in \hat{C} \mid \argmax_{i \in [1,2]} \mathcal{K}(\mathcal{P}_G, \mathcal{P}_{\hat{G}_i}) = j \}$, does not choose the best split among many possible splits, but it simply divides the set $\hat{C}$ into two subsets based on the two largest core clusters $\hat{G}_1,\hat{G}_2$. In other words, the criterion is not an assessment function. Hence, H-$\mathcal{K}C$ does not produce the same dendrogram as that produced by AHC or DHC with $f$.

Although  H-$\mathcal{K}C$ does not minimize $f$,  the bisecting criterion provides an upper-bound for $f$, i.e., $f(X,Y)\leq\tau$, where $\tau$ is the similarity threshold in $\mathcal{K}$-\texttt{psKC}, i.e. the default choice in the first line of H-$\mathcal{K}$C.

\end{proof}

\section{Data and Experiment Settings}
\subsection{Spatial Transcriptomics data}
\label{append: ST data}
The HER2 breast tumor dataset was downloaded from github: https://github.com/almaan/her2st. 
The dataset contains eight samples with pathologist-annotated labels, and we used the H1 sample to demonstrate the result.

Mouse hippocampus Slide-seq V2 data were downloaded from the Single Cell Portal SCP815 website: https://singlecell.broadinstitute.org/single\_cell/study/SCP815/sensitive-spatial-genome-wide-782. 

We used the file ``Puck\_200115\_08" in our study. The dataset contains approximately 23,000 genes and 53,000 spatial locations.

Begin with a spatial transcriptomics dataset, we normalized the raw molecular count matrix using the variance stabilizing transformation method called SCTransform \cite{choudhary2022comparison}. To select spatially variable genes, we apply SPARK \cite{sun2020statistical} in the HER2 breast tumor dataset and use SPARK-X \cite{zhu2021spark} in the mouse hippocampus Slide-seq V2 data, as conducted in previous studies. Then the cell spatial information  and the gene expression information are integrated to produce a graph. Then, a graph embedding scheme called the Weisfeiler-Lehman scheme \cite{shervashidze2011weisfeiler} converts a graph into a vector representation. The Weisfeiler–Lehman(WL) embedding used here is the
 same as that proposed by \cite{togninalli2019wasserstein}. For a graph $G=(V,E)$ with continuous attributes $a(v) \in \mathbb{R}^m$, the WL
 embedding at iteration $h>0$ is recursively defined as:

 \begin{equation*}
    a^{h+1}(v) = \frac{1}{2} \left( a^h(v) + \frac{1}{\text{deg}(v)} \sum_{u \in \mathcal{N}(v)} w((v, u)) \cdot a^h(u) \right).
\end{equation*}
When edge weights are not available, $w((u, v)$ is set to $1$. Using the recursive procedure described above, a WL-based graph embedding scheme that generates node embeddings from the attributes of the graphs can be proposed as
 \begin{equation*}
    \Phi(v) = \left[ a^0(v), \dots, a^h(v) \right]^\top.
\end{equation*}
Then, for a graph $G = (V, E)$, its mean embedding or its feature mean map is given as
\[
    \hat{\Phi}(G) = \frac{1}{|V|} \sum_{v \in V} \Phi(v) = \frac{1}{|V|} \sum_{v \in V} [a^0(v), \dots, a^h(v)]^\top.
\]

\subsection{Machine and hyperparameter settings}
The experiments are performed on a machine with 2.20GHz CPU and 16GB RAM. The search range of hyperparameter is provided in Table \ref{tab:parameters}. 
\begin{table}[h]
        \caption{Hyperparameter search range.}
    \label{tab:parameters}
    \centering
    \begin{tabular}{|c|c|} \hline
          & search range\\ \hline
      Isolation Kernel & $\psi\in\{4,6,8,16,24,32,48\}, t=200$\\ \hline
      \multirow{2}{*}{H-$\mathcal{K}C$}   & $\tau \in $ $[0.1, 0.5, 1,5,10,50,$\\&$100,500,1000]\times 10^{-4}$ $\rho=0.1$\\ \hline
      WL & $h=7$\\ \hline
      \multirow{2}{*}{SpaGCN} &  $histology\in\{true, false\}$\\
       & $ init \in \{``louvain", ``kmeans"\}$\\
      
      \hline
    \end{tabular}
\end{table}

\section{Additional Experiments}
\label{append: more exps}
\subsection{Ablation studies}
\label{append: ablation}

For the choice of distributional kernel, we consider recently introduced Isolation Distributional Kernel (IDK) \cite{ting2020isolation}, which origins from Isolation Kernel \cite{ting2018isolation} and has finite dimensional feature map.

In the first ablation study, we compare IDK with  Gaussian Distributional Kernel (GDK) in Figure \ref{fig: ablation}. This one needs to modify the affect lines 1, 7, 13 \& 17, where distributional kernel is used. In lines 1, 7, 13 \& 17, distribution kernel employs the same hyperparameter. As shown in Figure \ref{fig: ablation}. IDK successfully identifies three Gaussian distributions with varied densities and the clusters, in contrast to GDK which is unable to achieve this. GDK also can not correctly identify the two clusters on the bottom. In addition, after the first split, the leftmost Gaussian cluster is in the same subset with the L-shape cluster, instead of the middle Gaussian cluster.
\begin{figure}[h]
\centering
\begin{subfigure}
        {.45\linewidth}
      \centering   
\includegraphics[width=\linewidth]{pics_arxiv/obj2_HKC_0.98.png}
        \caption{IDK: $\wp=0.98$}
        \label{subfig: HKBC-IK}
    \end{subfigure} 
   \begin{subfigure}
        {.45\linewidth}
      \centering   
\includegraphics[width=\linewidth]{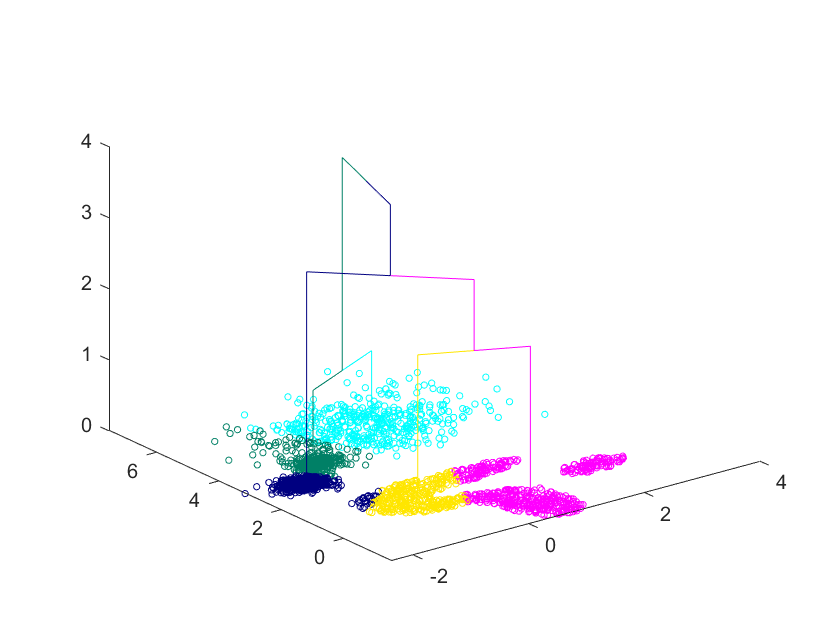}
        \caption{GDK: $\wp=0.86$}
        \label{subfig: HKBC-GK}
    \end{subfigure} 
    \caption{Results of IDK vs GDK in H-$\mathcal{K}C$ on the artificial dataset. $\wp$ is Dendrogram Purity (DP).}
\label{fig: ablation}
\end{figure}

The second ablation study examines the utility of the point-re-assignment (post-processing: line 14-18 in H-$\mathcal{K}C$).   The refinement or post-processing often provides tweaks at the edges of clusters. As a consequence, the results with and without post-processing are basically the same on the artificial dataset, as shown in Figure \ref{fig: ablation pp}.

\begin{figure}[h]
\centering
\begin{subfigure}
        {.45\linewidth}
      \centering   
\includegraphics[width=\linewidth]{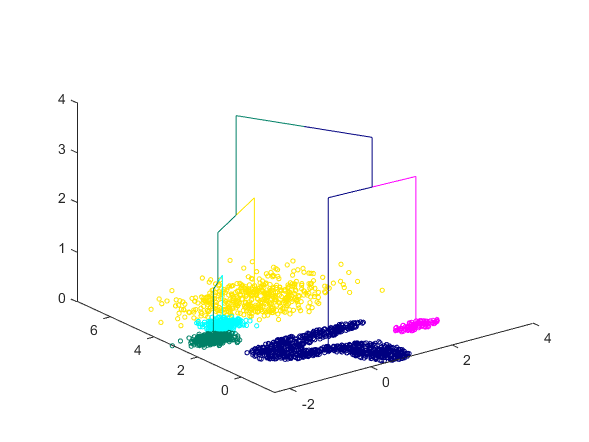}
        \caption{With post-processing:\\$\wp=0.98$}
        \label{subfig: HKBC-IK-refine}
    \end{subfigure} 
   \begin{subfigure}
        {.45\linewidth}
      \centering   
\includegraphics[width=\linewidth]{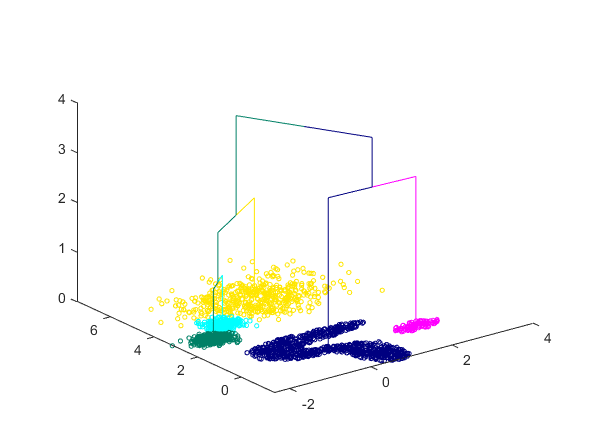}
        \caption{Without post-processing:\\$\wp=0.98$}
        \label{subfig: HKBC-GK-norefine}
    \end{subfigure} 
    \caption{Results of with and without post-processing.}
\label{fig: ablation pp}
\end{figure}

In the third ablation study, we consider three different clustering choices, $\mathcal{K}$-\texttt{psKC}, k-means \cite{kmeans} and DBSCAN \cite{DBSCAN_1996} in line 1 of H-$\mathcal{K}$C. The results are shown in Figure \ref{fig: cluster choice}. For k-means  and DBSCAN , we also consider their variants that uses Isolation Kernel (IK) \cite{ting2018isolation} as similarity measure. DBSCAN with IK can successfully identifies the clusters and produce the same dendrogram as that produced by $\mathcal{K}$-\texttt{psKC}.

In the last ablation study, we consider constructing the dendrogram with AHC. We replace line 4-12 of H-$\mathcal{K}C$ by using $h(X,Y)=\max_{C_i\in X,C_j\in Y}\mathcal{K}(\mathcal{P}_{C_i},\mathcal{P}_{C_j})$ for the merge of core clusters. The resultant dendrogram is exactly the same as that produced by H-$\mathcal{K}C$ in Figure \ref{subfig: HKBC-IK}. This result is consistent with our analysis in Proposition 1.

\begin{figure}[h]
\centering
     \begin{subfigure}
        {.3\linewidth}
      \centering   
\includegraphics[width=\linewidth]{pics_arxiv/obj2_HKC_0.98.png}
\caption{$\mathcal{K}$-\texttt{psKC}: $\wp=0.98$}
        \label{subfig: clsuter_pskc}
    \end{subfigure}
     \begin{subfigure}
        {.3\linewidth}
      \centering   
\includegraphics[width=\linewidth]{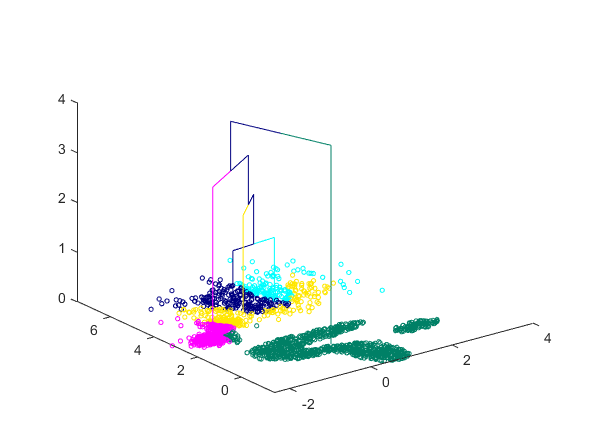}
        \caption{k-means: $\wp=0.76$}
        \label{subfig: cluster_kmeans}
    \end{subfigure}
\begin{subfigure}
        {.3\linewidth}
      \centering   
\includegraphics[width=\linewidth]{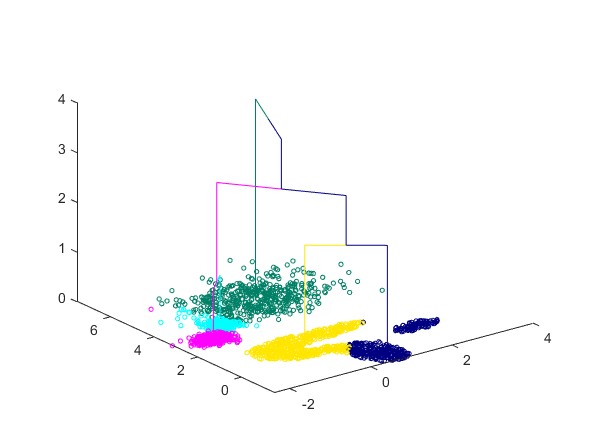}
        \caption{DBSCAN: $\wp=0.91$}
        \label{subfig: cluster_dbscan}
    \end{subfigure} \\
\begin{subfigure}
        {.3\linewidth}
      \centering   
\includegraphics[width=\linewidth]{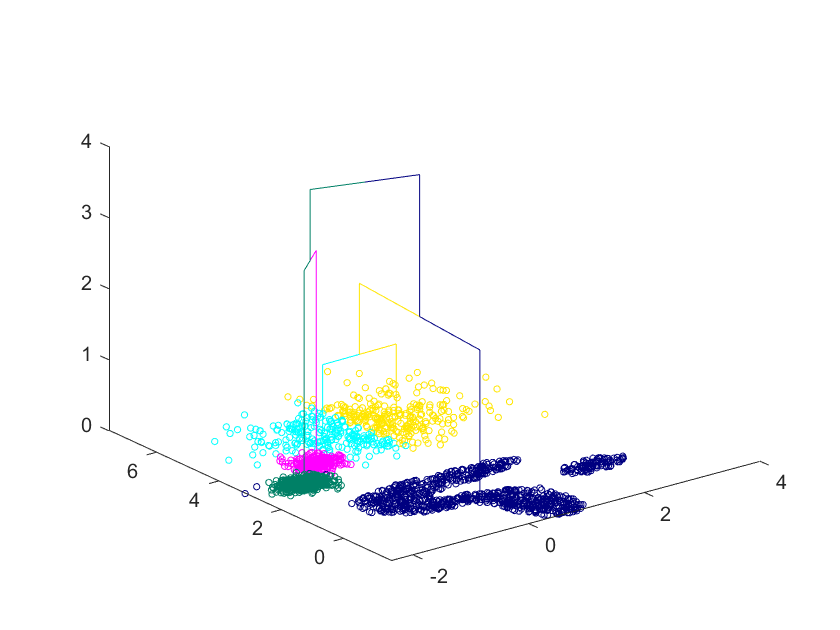}
        \caption{k-means (IK): $\wp=0.92$}
        \label{subfig: cluster_knkmeans}
    \end{subfigure} 
\begin{subfigure}
        {.3\linewidth}
      \centering   
\includegraphics[width=\linewidth]{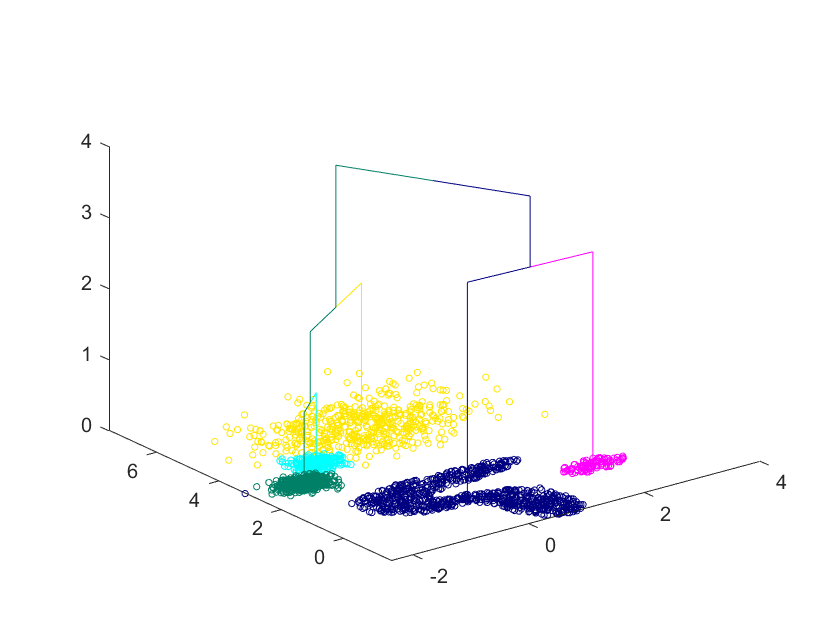}
        \caption{DBSCAN (IK): $\wp=0.98$}
        \label{subfig: cluster_kndbscan}
    \end{subfigure} 

    \caption{Results of different clustering method in line 1 of H-$\mathcal{K}$C.}
    %$\wp$ is the Dendrogram Purity.}
\label{fig: cluster choice}
\end{figure}

\subsection{SpecWRSC using psKC for initial clusters}
\label{appendix: SpecWRSC-psKC}

\begin{figure}[h]
      \centering   
\includegraphics[scale=0.3]{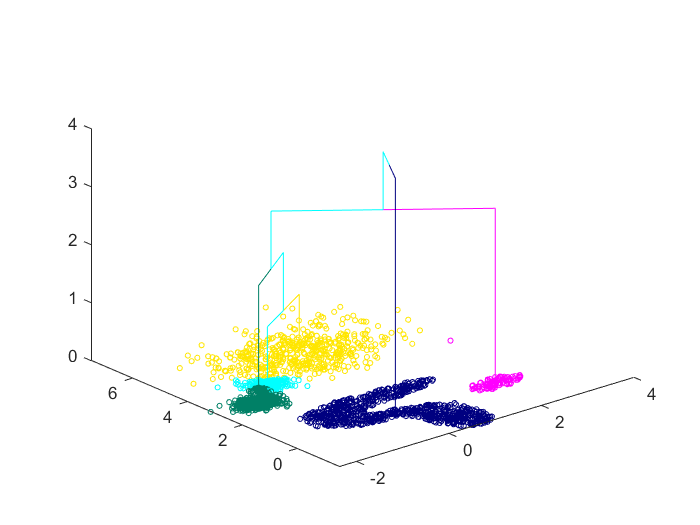}
        \caption{SpecWRSC using psKC for initial clusters: $\wp=0.95$}
        \label{fig: SpecWRSC-psKC}
\end{figure}

We present the results of SpecWRSC with psKC (instead of Spectral Clustering) as the initial clusters in Figure \ref{fig: SpecWRSC-psKC}.
\subsection{Scaleup test}

\label{append: scale}
The time complexity of H-$\mathcal{K}C$ is $\mathcal{O}(kn+s^2)$, where $n$ is the data size, $s$ is the data subset size and $k$ is the number of leaf nodes. 
Bisect-Kmeans has $\mathcal{O}(knR)$ time complexity, where $R$ is the number of repeated Kmeans.  SpecWRSC has time complexity of $G+\mathcal{O}(m\log^c n)$, where $G$ and $m$ are the time to build a graph and the number of edges in the graph, respectively. The time complexity of $G$ varies in the range $[\mathcal{O}(n), \mathcal{O}(n^2)]$, depending on the implementation. c is a constant greater than 1. The scaleup test results, shown in Figure \ref{fig: scaleup}, are consistent with their time
complexities.
\begin{figure}[h]
      \centering   
\includegraphics[scale=0.3]{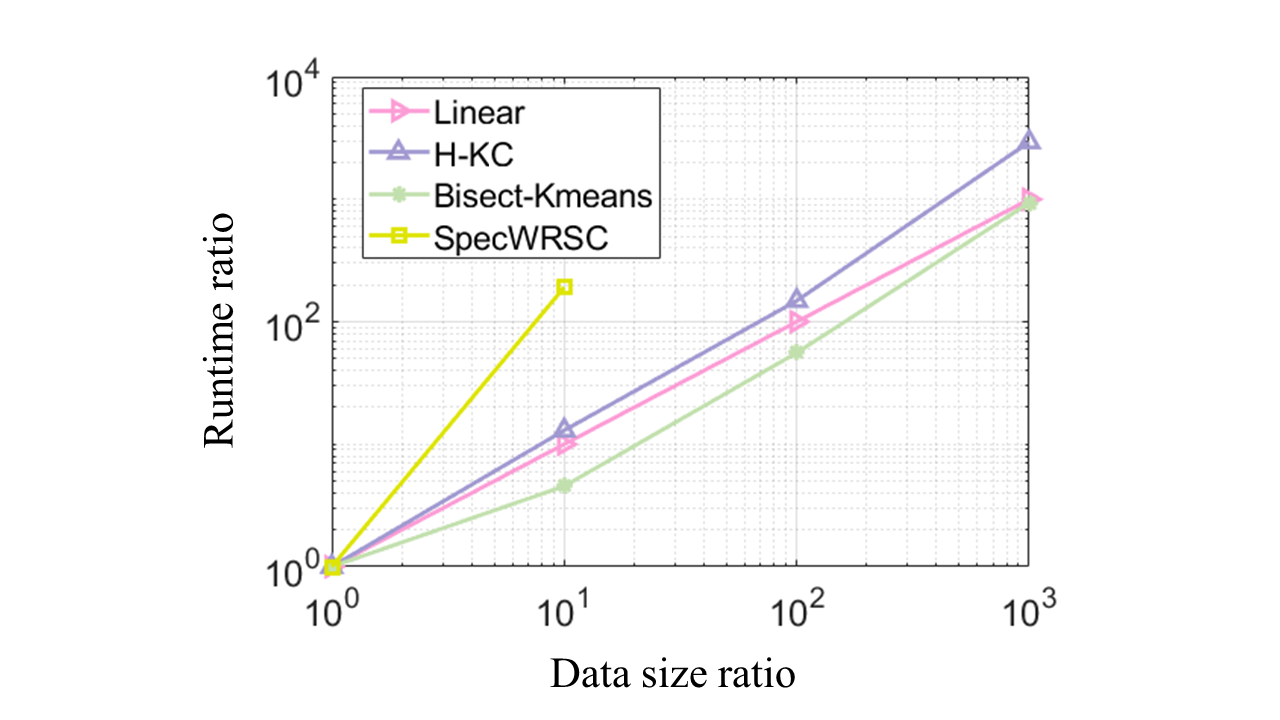}
        \caption{Scaleup test result on the same artificial dataset used in the main paper, where the data set size has 3000 points at data size ratio =1.}
        \label{fig: scaleup}
\end{figure}

% \begin{figure}[ht]
%       \centering   
% \includegraphics[scale=0.55]{pics_arxiv/spagcn.png}
%         \caption{SpaGCN's clustering result on the HER2 dataset.}
%         \label{fig: SpaGCN}
% \end{figure} 

\begin{figure}[h]
\centering
     \begin{subfigure}
        {.245\linewidth}
      \centering   
\includegraphics[width=\linewidth]{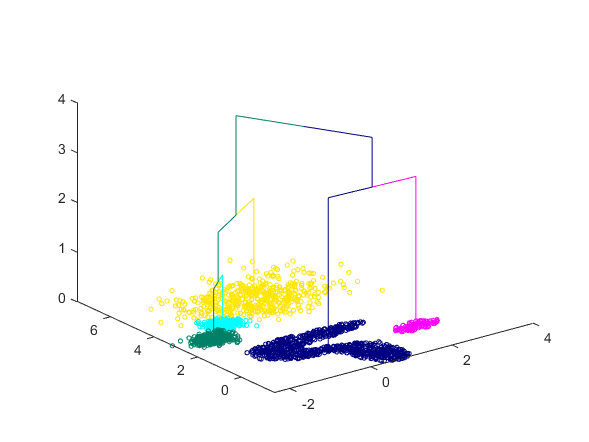}
\caption{$\tau=0.005$: $\wp=0.98$}
        \label{subfig: tau-1}
    \end{subfigure}
     \begin{subfigure}
        {.245\linewidth}
      \centering   
\includegraphics[width=\linewidth]{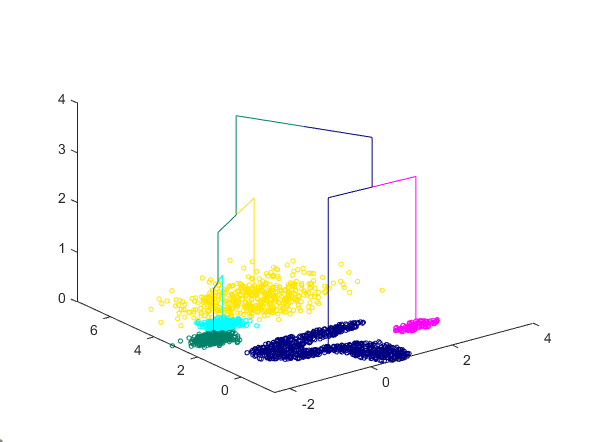}
        \caption{$\tau=0.01$: $\wp=0.98$}
        \label{subfig: tau-2}
    \end{subfigure}
\begin{subfigure}
        {.245\linewidth}
      \centering   
\includegraphics[width=\linewidth]{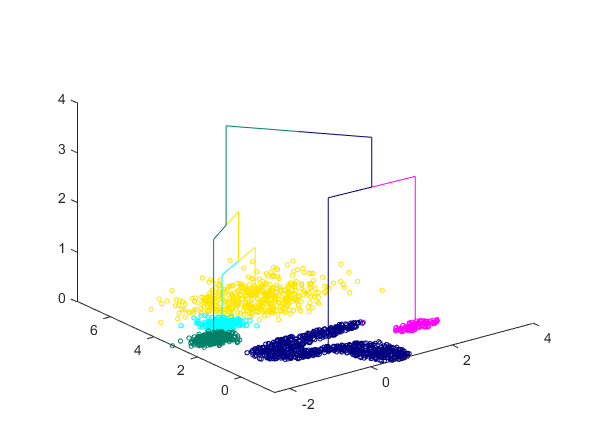}
        \caption{$\tau=0.015$: $\wp=0.98$}
        \label{subfig: tau-3}
    \end{subfigure} 
   \begin{subfigure}
        {.245\linewidth}
      \centering   
\includegraphics[width=\linewidth]{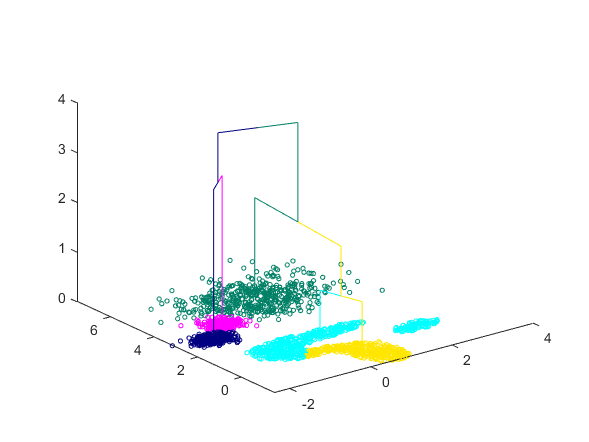}
        \caption{$\tau=0.02$: $\wp=0.92$}
        \label{subfig: tau-4}
    \end{subfigure} 
    \caption{Dendrograms of H-$\mathcal{K}C$ on the artificial dataset under different similarity thresholds $\tau$.}
    %$\wp$ is the Dendrogram Purity.}
\label{fig: sens tau}
\end{figure}

\begin{figure*}[h]
     \begin{subfigure}
        {.245\linewidth}
      \centering   
\includegraphics[width=\linewidth]{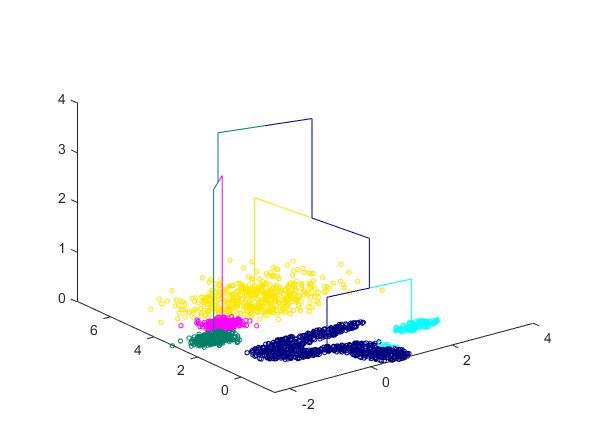}
        \caption{$s=1250$: $\wp=0.97$}
    \end{subfigure}
     \begin{subfigure}
        {.245\linewidth}
      \centering   
\includegraphics[width=\linewidth]{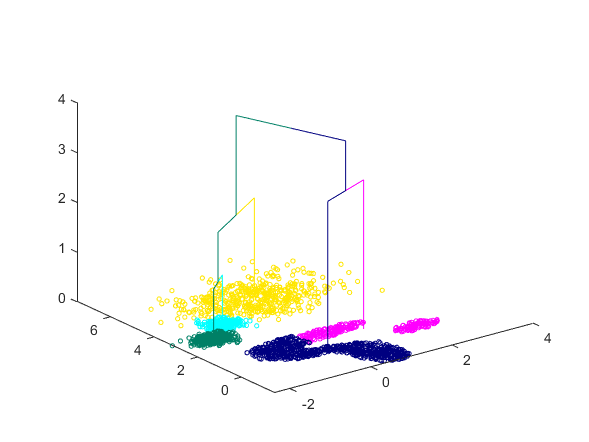}
        \caption{$s=1500$: $\wp=0.95$}
    \end{subfigure}
\begin{subfigure}
        {.245\linewidth}
      \centering   
\includegraphics[width=\linewidth]{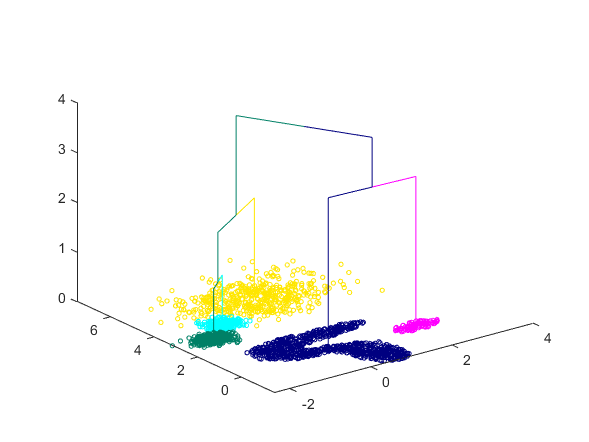}
        \caption{$s=1750$: $\wp=0.98$}
    \end{subfigure} 
   \begin{subfigure}
        {.245\linewidth}
      \centering   
\includegraphics[width=\linewidth]{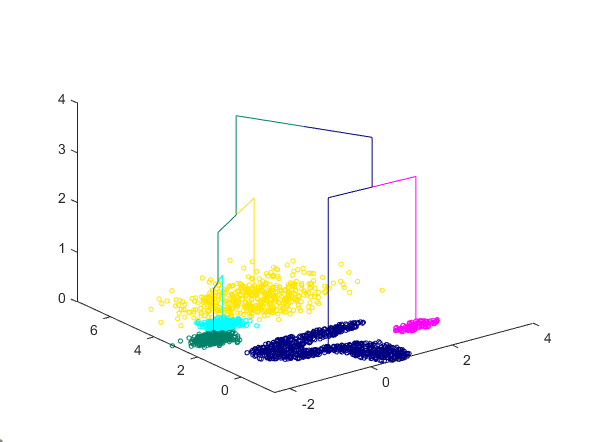}\caption{$s=2129$: $\wp=0.98$}
    \end{subfigure} 
    \caption{Dendrograms of H-$\mathcal{K}C$ on the artificial dataset under different data subset size $s$.} 
\label{fig: sens sample}
\end{figure*}

\begin{figure}[!h]
\centering
     \begin{subfigure}
        {.245\linewidth}
      \centering   
\includegraphics[width=\linewidth]{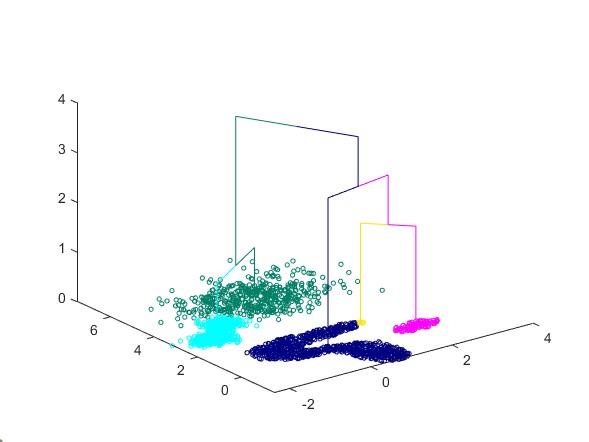}\caption{$\psi=40$: $\wp=0.85$}
    \end{subfigure}
     \begin{subfigure}
        {.245\linewidth}
      \centering   
\includegraphics[width=\linewidth]{pics_arxiv/0.980_psi60_tau0.01.png}
        \caption{$\psi=60$: $\wp=0.98$}
    \end{subfigure}
    \begin{subfigure}
        {.245\linewidth}
      \centering   
\includegraphics[width=\linewidth]{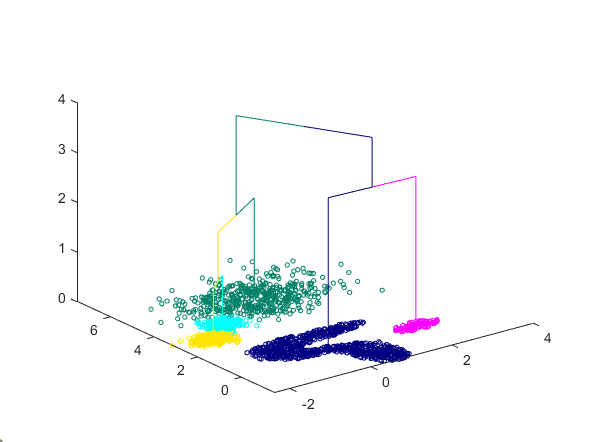}
        \caption{$\psi=80$: $\wp=0.98$}
    \end{subfigure}
    \begin{subfigure}
        {.245\linewidth}
      \centering   
\includegraphics[width=\linewidth]{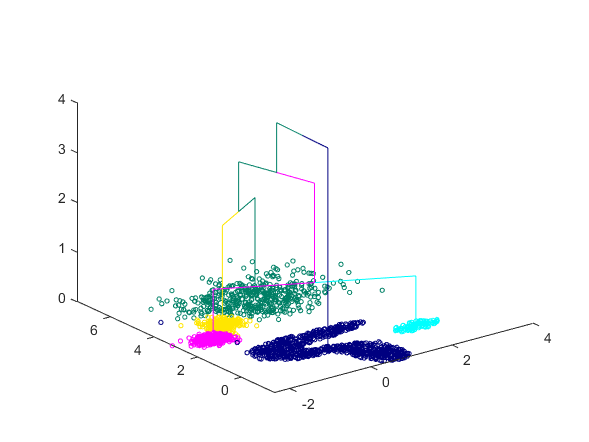}
        \caption{$\psi=100$: $\wp=0.97$}
    \end{subfigure}   \caption{Dendrograms of H-$\mathcal{K}C$ on the artificial dataset under different $\psi$ values.}
\label{fig: sens psi}
\end{figure}

\subsection{Flat clustering of SpaGCN on HER2}
\label{append: SpaGCN}
Here we use SpaGCN \cite{hu2021spagcn}, an end-to-end deep learning method, to perform flat clustering on the HER2 dataset, and it achieves NMI=0.45 and ARI=0.32. The clustering result is shown in Figure \ref{fig: SpaGCN}. The results of Bisect-Kmeans, SpecWRSC and H-$\mathcal{K}C$ are provided in Table \ref{tab: HER2 cluster}. Bisect-Kmeans is comparable to SpaGCN. H-$\mathcal{K}C$ outperforms the other methods. 

\begin{figure}[ht]
      \centering   
\includegraphics[scale=0.5]{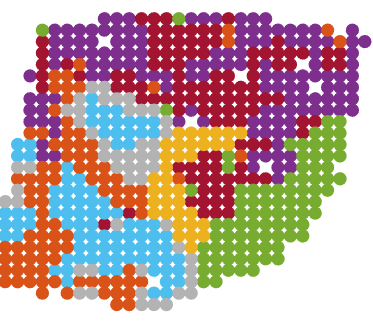}
        \caption{SpaGCN's clustering result on the HER2 dataset.}
        \label{fig: SpaGCN}
\end{figure}

\begin{table}[!ht]
    \centering
        \caption{Normalized Mutual Information (NMI) and Adjusted Rand Index (ARI) of different clustering algorithms. For DHC methods (Bisect-Kmeans, SpecWRSC and H-$\mathcal{K}C$), we use their leaf nodes as the final clustering result.}
    \label{tab: HER2 cluster}
    \begin{tabular}{|l|l|l|l|l|l|l|l|l|l|}
    \hline
        ~ & \multirow{2}{*}{SpaGCN} & Bisect- & \multirow{2}{*}{SpecWRSC} & \multirow{2}{*}{H-$\mathcal{K}C$} \\ 
        ~ & & Kmeans & &\\\hline
        NMI & 0.45 & 0.46 & 0.41 & 0.48 \\ \hline
        ARI & 0.32 & 0.39 & 0.26 & 0.43 \\ \hline
    \end{tabular}
\end{table}

\subsection{Sensitivity analysis}
Here we perform sensitivity analysis about the hyperparameters of H-$\mathcal{K}C$ and distributional kernel (Isolation Distributional Kernel) on the same artificial dataset. We adjust the hyperparameter within a small range near the optimal value $\tau=0.01,\psi=64$. 
\begin{itemize}
    \item For the similarity threshold $\tau$, it's more sensitive when $\tau$ is relatively large. It is robust in a small range around the optimal choice $\tau=0.1$, as shown in Figure \ref{fig: sens tau}.
    \item For the data subset size $s$, it is more robust when $s$ is large. When $s$ is small, the consequent core clusters are not representative enough to produce good clustering result, as shown in Figure \ref{fig: sens sample}.
    \item For the hyperparameter $\psi$ in Isolation Distributional Kernel (IDK), $\psi$ is a data-dependent hyperparameter, it is robust in a small range around the optimal choice $\psi=64$ (from 48 to 80), as shown in Figure \ref{fig: sens psi}.
\end{itemize}

\subsection{Comparison with more baselines}
\label{append: more baselines}

In addition to SpecWRSC and Bisect-Kmeans, we include comparison with more baselines (DHC methods DIANA \cite{DIANA-1990}, HDP \cite{zhu2022hierarchical} and AHC method SCC \cite{scc}) on additional 13 datasets. We report dendrogram purity in Table \ref{tab: baseline}. 

% The corresponding Critical Difference (CD) diagram of the post-hoc Nemenyi test is provided in Figure \ref{fig: cd diagram}. H-$\mathcal{K}C$ has the highest average ranking. There is no critical difference between H-$\mathcal{K}C$ and SCC.
% \begin{figure}[h]
%       \centering   
% \includegraphics[scale=0.35]{pics_arxiv/cdd.png}
%         \caption{Critical difference diagrams.}
%         \label{fig: cd diagram}
% \end{figure}

\begin{table}[H]
  % \small
    \centering    \caption{Hierarchy clustering results in terms of dendrogram purity. \textcolor{blue}{NC} indicatesthat the run could not be completed within one day.}
    \label{tab: baseline}
    		\renewcommand{\arraystretch}{1}
		\setlength{\tabcolsep}{3.pt}
  %   \resizebox{\textwidth}{!}{

  \begin{tabular}{l|c|cccc|c}
    \hline 
     & \multicolumn{5}{c|}{DHC} & \multicolumn{1}{c}{AHC} \\ \cline{2-7}
     
     & \multirow{2}{*}{H-$\mathcal{K}C$} & Bisect-& \multirow{2}{*}{DIANA}& \multirow{2}{*}{SpecWRSC}& \multirow{2}{*}{HDP}& \multirow{2}{*}{SCC}\\ 
     Dataset   &     &Kmeans &              & & &               \\ \hline
     ALLAML    & .73 & .71 & .63            & .72 & .73 & .74 \\ \hline
     LSVT      & .74 & .64 & .59            &.70    & .65 &.66    \\
     Wine      & .95 & .84 & .42            &.95    & .83 &.95    \\
     Seeds     & .87 & .84 & .41            &.88   & .82 &.85    \\ 
     Musk      & .57 & .54 & .51            &.57    & .54 &.54    \\ 
     WDBC      & .90 & .88 & .58            &.88   & .84 &.92    \\ 
     LandCover & .55 & .53 & .22            &.49   & .43 &.60    \\ 
     Banknote  & .97 & .63 & .51            &.60    & .98 &.90    \\ 
     Spam      & .84 & .68 & .54            &.70   & .57 &.68    \\ 
     ImageNet-10 & .84 & .80 & .10            &.78   & .84 &.86    \\ 
     STL-10    & .63 & .59 & .10            & .50   & .59 &.59    \\ 
     CIFAR-10  & .66 & .64 & .10            &  \textcolor{blue}{NC}   & .64   &.63    \\ 
     MNIST     & .54 & \textcolor{blue}{NC} & .69 &   \textcolor{blue}{NC}      & \textcolor{blue}{NC}   &.39    \\ 
     Covertype & .49 & .46 &  \textcolor{blue}{NC}& \textcolor{blue}{NC} &\textcolor{blue}{NC} & .45  \\ \hline
     Avg.ranking & 1.50 & 3.43 & 5.43 & 3.36 & 3.5& 2.29\\
    \hline
\end{tabular}
\end{table}

% \subsection{Code availability in the public domain}
% All the codes of our algorithm and data are provided in the attached zip file. Note that the codes of the CKMM objective-based algorithms \cite{cohen2017hierarchical,naumov2021objective} are not available in the public domain. Therefore, they could not be used for comparison.

%  \section{On the use of Large Language Models}

% While LLMs are leveraged to polish the language of this paper, the human authors exclusively developed and validated all substantive intellectual content, including the conceptual 

\bibliographystyle{alpha}
\bibliography{references}

\end{document}